%% file: Main.tex
\documentclass[twocolumn, a4paper,10pt]{article}
\pdfoutput=1
\usepackage[top=2.5cm, bottom=2.5cm, left=2.0cm, right=2.0cm,
columnsep=0.8cm]{geometry}

\usepackage{cite}
\usepackage{amsmath,amssymb,amsfonts}
\usepackage{algorithmic}
\usepackage{graphicx}
\usepackage{textcomp}
\usepackage{xcolor}
\def\BibTeX{{\rm B\kern-.05em{\sc i\kern-.025em b}\kern-.08em
    T\kern-.1667em\lower.7ex\hbox{E}\kern-.125emX}}
    
\usepackage{mathrsfs} 
\usepackage{tcolorbox,todonotes}     
\usepackage{xspace} 
\usepackage{amsfonts,amsmath,amsthm} 
\usepackage{algorithm}
\usepackage{algorithmic}
\usepackage{url}
\usepackage{ctable}
\usepackage{subfigure}
 
\newtheorem{theorem}{Theorem} 
\newtheorem{lemma}{Lemma}
\newtheorem{definition}{Definition} 
 \newtheorem{proposition}[theorem]{Proposition}
    
\newcommand{\bfA}{\mathbf{A}} 
\newcommand{\bfB}{\mathbf{B}}

\newcommand{\bfU}{\mathbf{U}}
\newcommand{\bfV}{\mathbf{V}}

\newcommand{\bfa}{\mathbf{a}} 
\newcommand{\bfb}{\mathbf{b}} 
\newcommand{\bfc}{\mathbf{c}}

\newcommand{\bfu}{\mathbf{u}}

\newcommand{\bfz}{\mathbf{z}}

\newcommand{\bfx}{\mathbf{x}}

\newcommand{\bfy}{\mathbf{y}}

\newcommand{\cR}{\mathcal{R}}

\newcommand{\OO}{{\mathcal O}}

\newcommand{\RR}{{\mathcal R}}

\newcommand{\ar}[2]{r}

\newcommand{\defproblem}[3]{
  \vspace{1mm}
\noindent\fbox{
  \begin{minipage}{.45\textwidth}
  \begin{tabular*}{\textwidth}{@{\extracolsep{\fill}}lr} \textsc{#1} \\ \end{tabular*}
  {\bf{Input:}} #2  \\
  {\bf{Task:}} #3
  \end{minipage}
  }
  \vspace{1mm}
}

 \newcommand{\bmfgfr}{\ProblemName{$\Bbb{F}_2$-MF}}

\newcommand{\rclustering}{{\sc Binary Constrained Clustering}\xspace}

\newcommand{\bmfbrfull}{{\sc  Boolean Matrix Factorization}\xspace}
\newcommand{\bmfbr}{{\sc BMF}\xspace}

\newcommand{\kmclust}{{\sc   $k$-Means Clustering}\xspace}

\newcommand{\argmin}{{\rm argmin}\xspace}
\newcommand{\rank}{{\rm rank}\xspace}
\newcommand{\Brank}{{\rm Boolean }\text{{\rm -rank}}\xspace}

\newcommand{\hdist}{d_H}
\newcommand{\Hdist}{{\operatorname{cost}}}

 \newcommand{\pname}{\textsc}
\newcommand{\ProblemFormat}[1]{\pname{#1}}
\newcommand{\ProblemIndex}[1]{\index{problem!\ProblemFormat{#1}}}
\newcommand{\ProblemName}[1]{\ProblemFormat{#1}\ProblemIndex{#1}{}\xspace}

  \newcommand{\probgenfull}[2]{\ProblemName{GF($#1$)-Matrix $\ell_{#2}$-norm Factorization}}
\newcommand{\probgen}[2]{\ProblemName{$\Bbb{F}_{#1}$-$\ell_{#2}$-MF}}

\newcommand{\algoclus}[2]{{\sf ConClustering($#1,#2$)}}
\newcommand{\algolargeclus}[2]{{\sf LargeConClustering($#1,#2$)}}

\newcommand{\clustering}[2]{\ProblemName{Constrained $(#1,#2)$-Clustering}}

\begin{document}

\title{Boolean and $\Bbb{F}_p$-Matrix Factorization: From Theory to Practice\thanks{The research received funding from the Research Council of Norway via the project BWCA (grant no. 314528) and IIT Hyderabad via Seed grant (SG/IITH/F224/2020-21/SG-79).  The work is conducted while Anurag Patil and Adil Tanveer were students at IIT Hyderabad.}}

\author{Fedor Fomin\footnote{Department  of Informatics, University of Bergen,  Norway} \and Fahad Panolan\footnote{Department of Computer Science and Engineering, IIT Hyderabad, India} \and  Anurag Patil\footnote{EdgeVerve Systems Limited,  Bengaluru,  India} \and Adil Tanveer\footnote{Amazon,  Chennai,  India}}

%

\date{}





\maketitle

\begin{abstract}  
Boolean Matrix Factorization (BMF) aims
 to find an approximation of a given binary matrix as the Boolean product of two low-rank binary matrices.
Binary data is ubiquitous in many fields, and representing data by binary matrices is common in medicine, natural language processing, bioinformatics, computer graphics, among many others.  Factorizing a matrix into low-rank matrices is used to gain more information about the data, like discovering relationships between the features and samples, roles and users, topics and articles, etc. In many applications, the binary nature of the factor matrices could enormously increase the interpretability of the data.

Unfortunately, BMF is computationally hard  and heuristic algorithms are used to compute Boolean factorizations. Very recently, the theoretical breakthrough was obtained independently by two research groups. Ban et al. (SODA 2019) and Fomin et al. (Trans. Algorithms 2020) show that BMF admits an efficient polynomial-time approximation scheme (EPTAS). However, despite the theoretical importance, the high double-exponential dependence of the running times from the rank makes these algorithms unimplementable in practice.  The primary research question motivating our work is whether the theoretical advances on BMF could lead to practical algorithms.

The main conceptional contribution of our work is the following.  While EPTAS for BMF is a purely theoretical advance, the general approach behind these algorithms could serve as the basis in designing better heuristics. We also use this strategy to develop new algorithms for related $\Bbb{F}_p$-Matrix Factorization. Here, given a matrix $\bfA$ over a finite field GF($p$) where $p$ is a prime, and an integer $r$,  our objective is to find a matrix $\bfB$ over the same field with GF($p$)-rank at most $r$ minimizing some norm of $\bfA-\bfB$.   
Our empirical research on synthetic and real-world data demonstrates the advantage of the new algorithms over previous works on BMF and $\Bbb{F}_p$-Matrix Factorization. 
\end{abstract}


\section{Introduction}

\input{introduction0}
\input{introduction1}

\input{algo}

\input{exp}

\input{expsup}

\section{Conclusion}

In this work we designed  heuristic algorithms for BMF and $\Bbb{F}_p$-Matrix Factorization that are inspired by the theoretical algorithms for the same.  Even though our algorithms have less error compared with the benchmark algorithms we considered,  the later run faster as they are truely polynomial time algorithms.  It is interesting research direction to improve the running time of the algorithm along with obtaining less error.

\bibliographystyle{plain}
\bibliography{book_pc,pca_with_outliers,k-clustering}



\end{document}

%% file: introduction0.tex


 Low-rank matrix approximation (matrix factorization) is a widely used method of compressing a matrix by reducing its dimension. It is an essential component of various data analysis techniques, including Principal Component Analysis (PCA), 
  the most popular and successful techniques used for dimension reduction in data analysis and machine learning \cite{pearson1901liii,hotelling1933analysis,eckart1936approximation}.
Low-rank matrix approximation is also a common tool  in 
 factor analysis for extracting latent features from data \cite{spearman1961general}.  
    
In the low-rank matrix approximation problem,  
we are given an   $m\times n$ real-valued matrix $\bfA$,  and the objective is to 
approximate $\bfA$ by a product of two low-rank matrices, or factors, $\bfU\cdot \bfV$, where $\bfU$ is a  $m\times r$ and  $\bfV$ is a  $r\times n$ matrix, and $r\ll m,n$. 
Equivalently, for an input  $m\times n$ data matrix $\bfA$ and $r\in {\mathbb N}$,   we seek an $m\times n$ matrix $\bfB$ of rank $r$ that 
approximates $\bfB$.
%
%
%
%
%
 By the Eckart-Young-Mirsky theorem,  best low-rank approximation could be found  via Singular Value Decomposition (SVD)~\cite{eckart1936approximation,MR0114821}.  However, SVD works only when no constraints are imposed on factor matrices $\bfU$ and  $\bfV$, and approximation is measured by the  Frobenius norm of $\bfA-\bfU\cdot \bfV$. 
In many application with binary data  when factorization is used as a pre-processing step or dimension reduction, it could be desirable to  run subsequent methods on  binary inputs.  Also in certain application domains  binary matrices are more interpretable \cite{ijcai2020-685}. 
However,   the desire  to ``keep the data binary''  makes the problem of factorization way more computationally challenging. Similar situation occurs with factorizing matrices  over   a finite field GF($p$). 

%
%
 
The large number of applications requiring Boolean or binary matrix factorization has given raise to many interesting heuristic algorithms for solving these computationally hard problems \cite{fu2010binary,Shen2009,Jiang2014,Koyuturk2003,DBLP:conf/icde/LuVA08,HessMP17}. In the theory community, also several algorithms for such problems were developed, including efficient polynomial-time approximation schemes (EPTAS) \cite{BanBBKLW19,FominGLP020}.  However, it seems that all these exciting developments in theory and practice occur in different universes. Besides a notable exception \cite{Kumar19f}, the ideas that were useful to advance the algorithmic theory of BMF do not find their place in practice.
 This bring us to the following question, which is the main motivation of our study.  
     \begin{tcolorbox}[colback=green!5!white,colframe=blue!40!black]
Could the ideas behind the theoretical advances on BMF be useful for practical algorithms?
 \end{tcolorbox} 
   
There is no immediate  answer to this question. The algorithms developed in  \cite{BanBBKLW19,FominGLP020} are  rather impractical due to tremendous exponential terms in the running times. See also the discussion in Section~4.3 of   \cite{ijcai2020-685}. However, as we demonstrate, at least of the ideas from \cite{BanBBKLW19,FominGLP020} could be extremely useful and for practical algorithms too.     
  
\paragraph*{Boolean and $\Bbb{F}_p$-Matrix Factorization}  
We   consider two low-rank matrix approximation problems.
Our first problem is  \bmfbrfull (\bmfbr). 
Let $\bfA$ be a binary $m\times n$ matrix. We consider the elements of $\bfA$ to be \emph{Boolean} variables. 
The \emph{Boolean rank} of $\bfA$ is the minimum $r$ such that $\bfA=\bfU \bfV$ for a Boolean $m\times r$ matrix $\bfU$ and a Boolean $r\times n$ matrix $\bfV$, where the product is Boolean. That is,  the logical $\wedge$ plays the role of multiplication and $\vee$ the role of sum. 
Thus the  matrix product is over the Boolean semi-ring $({0, 1}, \wedge, \vee)$. This can be equivalently expressed
as the  normal matrix product with addition defined as $1 + 1 =1$. Binary matrices equipped with such algebra are called \emph{Boolean matrices}. In  \bmfbr, the  objective is 
 
 \begin{eqnarray}\label{eq_PCA_b}
\text{ minimize } \|\bfA-\bfB\|_{0}  \\ 
\text{ subject to } \Brank(\bfB) \leq  r. \nonumber
\end{eqnarray}
Recall that $\|\cdot \|_0$ norm is the number of non-zero entries in the matrix. 

 In the second problem the matrices   are over a finite field GF($p$), where $p$ is a prime. 
The most common example of a finite field  GF($p$) is the set of  the integers $\mod p$,  where $p\geq 2$ is a prime number.
The matrix norm is the entry-wise $\ell_q$-norm $\| \cdot \|_q$.  Recall that for matrix $\bfA$, its   $\ell_q$ matrix norm is defined as 
$||\bfA||_q=  ({\sum_{i=1}^m\sum_{j=1}^n |a_{ij}|^q})^{1/q}$.   
In particular,  $\ell_2$ matrix norm is the Frobenius norm.
Then in the \probgenfull{p}{q} (\probgen{p}{q}) problem, 
we are given an $m\times n$ matrix $\bfA$ over GF($p$) and $r\in{\mathbb N}$, and the objective is to find a  matrix $\bfB$ over GF($p$) optimizing 
\begin{eqnarray}\label{eq_PCA_q}
\text{ minimize } \|\bfA-\bfB\|_{q}  \\ 
\text{ subject to } \text{GF($p$)-}\rank(\bfB) \leq  r. \nonumber
\end{eqnarray}
Here, $\text{GF($p$)-}\rank(\bfB)$ is the rank of the matrix $\bfB$ over field  GF($p$). 
Thus   the entries of the approximation matrix $\bfB$ in  \eqref{eq_PCA_q} should be integers from $\{0,\dots, p-1\}$ and  the arithmetic operations defining the rank of matrix $\bfB$ are over integers modulo $p$. 
The special case of \eqref{eq_PCA_q}  
  when $p=2$ and $q=1$ is  
the  \bmfgfr problem.
Let us remark that when the matrices are binary, the choice of the norm $\|\cdot \|_0$,  $\|\cdot \|_1$, or $\|\cdot \|_q$, for $q>1$, does not make any difference. For GF($p$) with $p>2$, the choice of the norm is essential.    
The difference of   \bmfgfr and  \bmfbr is in the definition of  the rank of $\bfB$.  This is a significant difference because the  $\text{GF(2)-}\rank$ is computable in polynomial time, say by the Gaussian elimination, and computing the $\Brank$ of a matrix is already an NP-hard problem. 
 We design new algorithms for \probgen{p}{q} and \bmfbr and test them on synthetic and real-world data. 


\paragraph*{Related work} Both problems are well-known in Machine Learning and Data Mining communities. Since 
 \bmfbr  was studied in different communities,  in the literature it also appears under different names like \textsc{Discrete Basis Problem}  \cite{MiettinenMGDM08} or 
\textsc{Minimal Noise Role Mining Problem}  \cite{VaidyaAG07,LuVAH12,Mitra:2016}.

The GF(2), and more generally, GF($p$) models find applications for
Independent Component Analysis in signal processing~\cite{GutchGYT12,PainskyRF16,Yeredor11}, latent  semantic analysis \cite{berry1995using}, or pattern discovery for gene expression~\cite{Shen2009}. \probgen{p}{q}  is an essential tool in dimension reduction for high-dimensional data with binary attributes \cite{Koyuturk2003,Jiang2014}.   \bmfbr has found applications in data
mining such as  topic models, association rule mining, and database tiling 
\cite{BelohlavekV10,DanHJWZ15,LuVAH12,MiettinenMGDM08,DBLP:conf/kdd/MiettinenV11,DBLP:conf/icde/Vaidya12}.  The recent survey \cite{ijcai2020-685} provides a  
 concise overview of the current theoretical and practical algorithms proposed for  \bmfbr.

The constraints imposed on the properties of factorization in \eqref{eq_PCA_q} and \eqref{eq_PCA_b}
make the problems computationally intractable.  Gillis et al.~\cite{GillisV15} proved that  \bmfgfr is 
 NP-hard already for $r=1$. 
Since the problems over finite fields are computationally much more challenging, 
 it is not surprising that most of the practical approaches for handling these problems are heuristics ~\cite{fu2010binary,Shen2009,Jiang2014,Koyuturk2003,DBLP:conf/icde/LuVA08}.  
  
  Another interesting trend  in the study of low-rank matrix approximation problems develops in algorithmic theory. A number of algorithms with guaranteed performance were developed for  \probgen{p}{q}, \bmfgfr, and \bmfbr.  Lu et al.  \cite{DBLP:conf/icde/LuVA08} gave a formulation of \bmfbr as an integer programming problem with exponential number of variables 
and constraints. Parameterized algorithms for \bmfgfr  and \bmfbr were obtained  in \cite{FominGP20}.  A  number of approximation algorithms  were developed, resulting in efficient  polynomial time approximation schemes (EPTASes)  obtained in 
\cite{BanBBKLW19,FominGLP020}. 
Parameterized and approximation algorithms from \cite{FominGP20,BanBBKLW19, FominGLP020} are mainly of theoretical importance and are  not implementable due to tremendous running times. 
Bhattacharya et al.~\cite{Bhattacharya19} extended ideas in \cite{BanBBKLW19,FominGLP020} to obtain a 4-pass streaming algorithm which computes a $(1+ \varepsilon)$-approximate \bmfbr.  Kumar et al.~\cite{Kumar19f} designed bicriteria approximation algorithms for \bmfgfr. Except the work of  Kumar et al.~\cite{Kumar19f}, none of the above theoretical algorithms were implemented.

\paragraph*{General overview of the main challenges}
The starting point of our algorithms  for \probgen{p}{q} and \bmfbr are the approximation algorithms developed in \cite{BanBBKLW19,FominGLP020}. The general ideas from these papers  are similar, here we follow  \cite{FominGLP020}.  They develop algorithms for  \bmfbr  and \bmfgfr but generalizations to  \probgen{p}{q} is not difficult. 

The two basic  steps of the approach of \cite{FominGLP020} are the following.  First encode the matrix factorization problem  as a clustering problem with specific constraints on the clusters' centers.   Then use  sampling  similar to 
the sampling used for vanilla $k$-means of \cite{KumarSS10}   for constructing a good approximation.  Implementation of each of these steps is a challenge, if possible at all. In the first step, encoding matrix factorization with rank $r$ results in constrained clustering with $2^{r}$ centers. But what makes the situation even worse is the second step. To obtain a reasonable guaranteed estimate for constrained clustering, one has to take exponentially many samples (exponential in $2^r$ and the error parameter $\varepsilon$), which is the bottleneck in the algorithm's running time.

%
%
%


The first idea that instead of sampling, we implement a simple procedure similar to Lloyd's heuristic for clustering \cite{Lloyd82} adapted for constrained clustering. 
 This is a simple and easily implementable idea. However,  due to the power of encoding the matrix factorization as clustering, 
  in many cases, our algorithm significantly outperforms previously known, sometimes quite involved, heuristics. The problem is that  this strategy works only for very small values of rank $r\leq 5$. This is because the factorization problem is encoded as the problem with $2^r$-clustering and the time required to construct the corresponding instance of clustering is of order  $2^{2r}$. For larger values of $r$ we need to develop a new algorithm that non-trivially uses the algorithm for small rank $r$. 

%% file: introduction1.tex
\subsection{Our methods}
Our algorithm for small values of $r$, follows the steps similar to Lloyd's algorithm or 
 the closely related $k$-means clustering algorithm. We start from some partition of the columns of the matrix. Then the algorithm repeatedly finds the centroid of each set in the partition and then re-partitions the input according to which of these centroids is closest. However, while for $k$-means clustering, the centroid is selected as the vector minimizing the sum of distances to all vectors in the cluster, in our case, the set of centroids should also satisfy a specific property. 

More precisely, in  the \kmclust\ problem we are given a set of points $X\subseteq {\mathbb R}^m$ and $k\in {\mathbb N}$, and the objective is to find $k$ center points $\bfc_1,\ldots,\bfc_k\in {\mathbb R}^m$ such that $\sum_{x\in X} \min_{i} ||x-\bfc_i||_2^2$ is minimized. For a set of $k$ centroids $\bfc_1,\ldots,\bfc_k$, one can define $k$ clusters $X_1,\ldots,X_k$ such that 
their union is $X$ and ($*$) for any $x\in X_i$, $\bfc_i$ is one of the closest point to $x$. For a given set of clusters $X_1,\ldots,X_k$, the best centers $\bfc_1,\ldots,\bfc_k$ satisfying ($*$) can be obtained by computing the centroid of $X_i$ for all $i\in \{1,\ldots,k\}$. 
The $k$-means algorithm  starts with a random set of $k$ clusters $X_{1,1},\ldots, X_{1,k}$ of $X$ and then finds their centroids. Then using these centroids we find $k$ clusters $X_{2,1},\ldots, X_{2,k}$ satisfying ($*$). Then, again we compute a set of centroids for $X_{2,1},\ldots, X_{2,k}$ and so on. It is easy to verify that the ``cost of a solution'' in each iteration is at least as good as the previous iteration. This algorithm converges very fast and outputs very good solution in practice. 

In order to apply ideas similar to the $k$-means algorithm for \probgen{p}{q} and \bmfbr,  we use the  
``constrained'' version of clustering introduced  by Fomin et al.~\cite{FominGLP020}. 

A $k$-ary relation $R$ over $\{0,1\}$ is a set of binary $k$-tuples with elements from $\{0,1\}$. A $k$-tuple $t=(t_1,\dots, t_k)$ \emph{satisfies} $R$, if $t\in R$. 
\begin{definition}[Vectors satisfying $\cR$~\cite{FominGLP020}]
Let $\cR=\{R_1, \dots, R_m\}$ be a set of $k$-ary relations. We say that a set $C=\{\bfc_1, \bfc_2, \dots, \bfc_k\}$ of binary $m$-dimensional vectors  \emph{satisfies $\cR$},  
if 
 $(\bfc_1[i],\ldots,\bfc_k[i])\in R_i$ for all $i\in \{1,\ldots,m\}$.
\end{definition}

For example, for $m=2$, $k=3$, $R_1=\{(0,0,1), (1,0,0)\}$, and  $R_2=\{(1,1,1), (1,0,1), (0,0,1)\}$, the set of vectors 
\[
\bfc_1=\left(
\begin{array}{c}
0\\
1\\
\end{array}
\right) , \, 
\bfc_2=\left(
\begin{array}{c}
0\\
1\\
\end{array}
\right) , \,
\bfc_3=\left(
\begin{array}{c}
1\\
1\\
\end{array}
\right)   \]
satisfies $\cR=\{R_1,  R_2\}$ because  $(\bfc_1[1],\bfc_2[1],\bfc_3[1])=(0,0,1)\in   R_1$ and $({\bfc}_1[2],\bfc_2[2],\bfc_3[2])=(1,1,1)\in   R_2$. 

The \emph{Hamming distance} between two vectors $\bfx, \bfy\in\{0,1\}^m$, where $\bfx=(x_1,\ldots,x_m)^\intercal$ and $\bfy=(y_1,\ldots,y_m)^\intercal$, is $\hdist(\bfx,\bfy)=\sum_{i=1}^m |x_i-y_i|$. 
For a set of vectors $C$ and a vector $\bfx$, we define 
$\hdist(\bfx,C)=\min_{\bfc\in C}\hdist(\bfx,\bfc)$. 
Then, the problem \rclustering\ is defined as follows.

\defproblem{\rclustering (\sc BCC)}{A set $X\subseteq \{0,1\}^m$ of $n$ vectors, a positive integer $k$, and a set of $k$-ary relations
$\cR=\{R_1, \dots, R_m\}$. }{Among all  vector sets $C=\{\bfc_1,\ldots,\bfc_k\}\subseteq \{0,1\}^m$ satisfying $\cR$, find a set $C$ minimizing the sum 
$\sum_{\bfx\in X} \hdist(\bfx,C)$.}

 
The following proposition is from \cite{FominGLP020} and for completeness we give a (different) proof sketch here. 

\begin{proposition}[\cite{FominGLP020}]\label{prop:matrixFasRClust} For any instance $(\bfA,r)$ of  \probgen{2}{1} (\bmfbr) 
 one can  construct in time $\OO(m+n+2^{2r})$ an instance $(X,k=2^r,\cR)$ of 
 {\sc BCC} 
with the below property, where $X$ is the set of column vectors of $\bfA$: 
\begin{itemize} 
 \item for any $\alpha$-approximate solution $C$ of $(X,k, \cR)$ there is an algorithm that in  time $\OO(rmn)$ returns   an $\alpha$-approximate solution $\bfB$ of $(\bfA,r)$, and 
 \item for any $\alpha$-approximate solution $\bfB$ of $(\bfA,r)$,  there is an algorithm that in  time $\OO(rmn)$ returns   an  $\alpha$-approximate solution $C$ of $(X,k, \cR)$.
 \end{itemize}
%
%
\end{proposition}

\begin{proof}[Proof sketch]
First we prove the proposition for \probgen{2}{1}. 
Let $(\bfA,r)$ be the input instance of \probgen{2}{1}. Recall that $X$ is a the set of column vectors of $\bfA$ and $k=2^r$. 
Now, we explain how to construct the relations $\cR=\{R_1,\ldots,R_m\}$. Here, we will have $R_i=R_j$ for all $i,j\in \{1,\ldots, m\}$ and we denote this relation by $R$. The relation $R$ depends only on $r$. 
Let $S_0,S_1,\ldots,S_{k-1}$ be the distinct subsets of $\{1,\ldots,r\}$ listed in the non-decreasing order of its size. Each $x=(x_1,\ldots,x_r)\in \{0,1\}^r$  we correspond  a tuple $(y_{0},\ldots,y_{k-1})$ in $R$  as follows. For each $i\in \{0,\ldots,k-1\}$, we set $y_{i}=(\sum_{j\in S_i} x_j)\mod 2$. That is, $R$ contains $2^r$ tuples, one for each $x\in \{0,1\}^r$. This completes the construction of the output instance $(X,k,\cR)$ of {\sc BCC}.  See also an example of a construction after the proof.

Now, given a solution $C=\{\bfc_0,\ldots,\bfc_{k-1}\}$ to the instance $(X,k,\cR)$ of {\sc BCC}, we can construct a solution $\bfB$ to the instance $(\bfA,r)$ of \probgen{2}{1} as follows. For each $i\in \{1,\ldots, n\}$, let $\bfb_i=\argmin_{\bfc \in C} d_{H}(\bfa_i,\bfc)$, where $\bfa_i$ is the $i$th column vector of $\bfA$. Now, for all $i\in \{1,\ldots, n\}$, we set the $i$-th column of $\bfB$ to be $\bfb_i$. From the construction of the relations $\cR$, any vector in $C$ is a linear combination of $\{\bfc_1,\ldots,\bfc_r\}$. This implies that the rank of $\bfB$ is at most $r$.  

Now suppose $\bfB$ is a solution to the instance $(\bfA,r)$ of {\sc BCC}. Let $Q=\{\bfc_1,\ldots,\bfc_r\}$ be a (multi)set of $r$ column vectors in $\bfB$ such that each column vector in $\bfB$ is a linear combination of vectors in $Q$. Such a set $Q$ exists because the rank of $\bfB$ is at most $r$. Recall that $S_0,S_1,\ldots,S_{k-1}$ are the distinct subsets of $\{1,\ldots,r\}$ listed in the non-decreasing order of the subset sizes. For each $i\in \{0,\ldots,k-1\}$, let $\bfu_{i}=(\sum_{j\in S_i} \bfc_j)\mod 2$. Then, $\{\bfu_{0},\bfu_{1},\ldots,\bfu_{{k-1}}\}$ is a solution to $(X,k,\cR)$. 

For the proof when $(\bfA,r)$ is an instance of {\sc BMF}, in the above construction we replace the addition mod 2 operations with the logical $\vee$ operations. 
%
\end{proof}

%
%
%
%

Let us give an  example of constructing constraints for $r=3$. Here,  $S_0=\emptyset$, $S_1=\{1\}$, $S_2=\{2\}$, $S_3=\{3\}$, 
$S_4=\{1,2\}$, $S_5=\{1,3\}$, $S_6=\{2,3\}$, and $S_7=\{1,2,3\}$.
For each binary $3$-tuple $x=(x_1,x_2, x_3)$, we correspond a binary $8$-tuple from $R$.
 The $i$-th element of this tuple is  $ (\sum_{j\in S_i} x_j)\mod 2$. 
 For example, for
$x=(1,1,0)$, we have a tuple $(0,1,1,0, 0,1,1,0)$ in $R$. 
Thus, we construct the set of constraints 
$R=\{
(0,0,0,0,       0,0,0,0),
(0,1,0,0,       1,1,0,1),
(0,0,1,0,       1,0,1,1),\\
(0,0,0,1,       0,1,1,1),
(0,1,1,0,       0,1,1,0),
(0,1,0,1,       1,0,1,0),\\
(0,0,1,1,       1,1,0,0),
(0,1,1,1,       0,0,0,1)
\}$. 
%

Now for  any set of centers $\{\bfc_0,\ldots,\bfc_7\}$ satisfying the above relations,  $\bfc_0=\vec{0}$ and any $\bfc_i$ is a linear combination of $\{\bfc_1,\bfc_2,\bfc_3\}$.
For example, suppose  $\bfc_0,\ldots,\bfc_7$ are the columns of the following matrix.

\begin{equation*}
\left( \begin{array}{cccccccc} 
0 & 1 & 0 & 0 &       1 & 1 & 0 & 1 \\
0 & 0 & 0 & 1 &       0 & 1 & 1 & 1 \\
0 & 1 & 1 & 0 &       0 & 1 & 1 & 0 \\
0 & 0 & 1 & 1 &       1 & 1 & 0 & 0 \\
0 & 1 & 1 & 1 &       0 & 0 & 0 &1
\end{array} 
\right).
\end{equation*}
Let us note that each of the rows of the matrix is one of the $8$-tuples of $R$. 
Then $\bfc_0=\vec{0}$, $\bfc_4=(\bfc_1+\bfc_2)\mod 2$, $\bfc_5=(\bfc_1+\bfc_3)\mod 2$,  $\bfc_6=(\bfc_2+\bfc_3)\mod 2$, and $\bfc_7=(\bfc_1+ \bfc_2  +\bfc_3)\mod 2$. 
Thus the  GF($p$)-rank of this  matrix is at most $3$.  

\medskip



We remark that our algorithms and the algorithms of Fomin et al.~\cite{FominGLP020} are different. Both the algorithm uses 
Proposition~\ref{prop:matrixFasRClust} as the first step. Afterwards, Fomin et al.~\cite{FominGLP020} uses sampling methods and this step takes time double-exponential in $r$. But, we use a method similar to the Lloyd's algorithm in the case of small ranks. For the case of large ranks we use several executions of Lloyd's algorithm on top of our algorithm for small ranks. We    overview  our algorithms   below.


\paragraph*{Algorithms for small rank} 
Because of Proposition~\ref{prop:matrixFasRClust}, we know that {\sc BCC} is a general problem that subsumes 
{\sc BMF} and \probgen{2}{1}. 
Let $I=(X,k,{\mathcal R}=\{R_1,\ldots,R_m\})$ be an instance of 
{\sc BCC} 
  and $C=\{\bfc_1,\ldots,\bfc_k\}$ be a solution to $I$. In other words,  $C$ satisfies ${\mathcal R}$.  We  call $C$ to be the  {set of centers}. 
We  define the \emph{cost of  the solution} 
$C$ of $I$ to be $\Hdist(X,C)=\sum_{\bfx\in X} \hdist(\bfx,C)$. Given set $C$, there is a natural way we can partition the set of vectors $X$ into $k$ 
sets $X_1\uplus\cdots\uplus X_k$, where for each vector $\bfx$ in $X_i$, the closest to $\bfx$ vector from $C$ is  $\bfc_i$. That is, 

\begin{equation}\label{clsuersfromcenter}
\Hdist(X,C)=\sum_{i=1}^{k}\sum_{\bfx\in X_i}\hdist(\bfx,\bfc_i)
\end{equation}
%
We call  such partition \emph{clustering of $X$ induced by $C$} and refer to sets $X_1,\ldots,X_k$ as to  {\em clusters corresponding to   $C$}. That is, given a solution $C$, we can easily find the clusters such that the best possible set of centers for these clusters is~$C$. 

Next, we explain how we compute the best possible centers from a given set of clusters of $X$. For a  partition $X_1\uplus\cdots\uplus X_k$ of $X$, $i\in [m]$, and $(b_1,\ldots,b_k)\in R_i$, define 
\begin{equation}
\label{centerfromclusters}
f_i(b_1,\ldots,b_k)=\sum_{j=1}^k\sum_{\bfx\in X_j}|\bfx[i]-b_j|
\end{equation}
Now, the set $\{\bfc_1,\ldots,\bfc_k\}$ be such that for any $i\in \{1,2,\ldots,m\}$, $(\bfc_1[i],\ldots,\bfc_k[i])=\argmin_{b\in R_i} f_i(b)$. One can easily verify that the best possible set of centers for the clusters $X_1,\ldots,X_k$ is 
$\{\bfc_1,\ldots,\bfc_k\}$.  That is, for any set of centers $\{\bfc_1',\ldots,\bfc_k'\}$ satisfying $\cR$, 

\begin{equation}
\label{centerfromclustersproperty}
\sum_{j=1}^{k}\sum_{\bfx\in X_j}\hdist(\bfx,\bfc_j) \leq \sum_{j=1}^{k}\sum_{\bfx\in X_j}\hdist(\bfx,\bfc_j')
\end{equation}

Our algorithm for 
{\sc BCC} 
 works as follows. Initially we take a random partition $X_{0,1}\uplus\cdots\uplus X_{0,k}$ of $X$. Then, using \eqref{centerfromclusters}, we find a solution $C_1=\{\bfc_{1,1},\bfc_{1,2},\ldots,\bfc_{1,k}\}$. Then, we find clusters $X_{1,1},\ldots,X_{1,k}$ corresponding to $C_1$ (i.e., $C_1$ and $\{X_{1,1},\ldots,X_{1,k}\}$ satisfies \eqref{clsuersfromcenter}). This implies that 
 
 \begin{eqnarray}
\sum_{j=1}^{k}\sum_{\bfx\in X_{1,j}}\hdist(\bfx,\bfc_{1,j})&=&\Hdist(X,C_1) 
\label{eqn:aabb}
 \end{eqnarray}
 
Now, again using \eqref{centerfromclusters} and the partition $\{X_{1,1},\ldots,X_{1,k}\}$, we find a solution $C_2=\{\bfc_{2,1},\bfc_{2,2},\ldots,\bfc_{2,k}\}$. Thus, by the property mentioned in \eqref{centerfromclustersproperty}, we have that 

\begin{equation}
\label{eqn:aacc}
\sum_{j=1}^{k}\sum_{\bfx\in X_{1,j}}\hdist(\bfx,\bfc_{2,j}) \leq \sum_{j=1}^{k}\sum_{\bfx\in X_{1,j}}\hdist(\bfx,\bfc_{1,j}) 
\end{equation}

Because of \eqref{eqn:aabb}, \eqref{eqn:aacc}, and the fact that $\Hdist(X,C_2)\leq \sum_{i=1}^{k}\sum_{\bfx\in X_{1,i}}\hdist(\bfx,\bfc_{2,i})$, we have that $\Hdist(X,C_2) \leq \Hdist(X,C_1)$. 
If $\Hdist(X,C_2) < \Hdist(X,C_1)$, we continue the above steps using the partition $X_{1,1}\uplus \ldots \uplus X_{1,k}$ and so on. 
Our algorithm continues this process until the cost of the solution converges. 


Our algorithm works well when $r$ is small (i.e., our algorithm on the output instances of Proposition~\ref{prop:matrixFasRClust}). 
Notice that $2^{2r}$ is a lower bound on the running time of the above algorithm when we use it for \probgen{2}{1} and \bmfbr (See Proposition~\ref{prop:matrixFasRClust}). For example, when $r=20$ the algorithm takes at least $2^{40}$ steps. So for large values of $r$, this algorithm is slow. 


\paragraph*{Algorithms for large rank}

For large $r$, we design new  algorithms for \probgen{p}{q} and \bmfbr which use  our base algorithm (the one explained above) for smaller values of rank. 
Here, we explain an overview of our algorithm for \bmfbr for large $r$. 
Let us use the term {\sf LRBMF} for the base algorithm for 
\bmfbr. 

Consider the case when $r=20$. Let $\bfA$ be the input matrix for \bmfbr. The idea is to split the matrix $\bfA$ into small parts and obtain approximate matrices of small rank (say $5$ or less) for all parts  using 
{\sf LRBMF} and merge these parts to get a matrix of rank at most $20$. 
Let $X$ be the set of columns of the input matrix $\bfA$. Suppose we partition the columns of $\bfA$ into four parts of almost equal size. Let $X_1,\ldots,X_4$ be these parts and let $\bfA_i$ be the matrix formed using columns of $X_i$ for all $i\in \{1,\ldots,4\}$. Let $\bfB_i$ be the output of  {\sf LRBMF} on the input $(\bfA_i,5)$ for all $i\in \{1,\ldots,4\}$. Then, by merging $\bfB_1,\ldots,\bfB_4$ we get a matrix of rank at most $20$. But this method did not give us good results because identical columns may be moved to different parts in $X_1,\ldots,X_4$. Thus, it is important that we  do this partition carefully. One obvious method is to use Lloyd's algorithm to get a partition of $X$ into four parts. But, unfortunately, even this method does not give us good results.

For our algorithm we use an iterative process to get a partition of $X$ where we use Lloyd's algorithm in each step. 
In the initial step we run Lloyd's algorithm on $(X,20)$ and let ${C}=\{c_1,\ldots,c_{20}\}$ be the set of output centers. Now we do an iterative process to partition $C$ with each block containing at most $5$ vectors. Towards that we run  Lloyd's algorithm on $(C,4)$. Let ${\cal Z}$ be the set of output clusters. If a cluster has size at most $5$, then that cluster is a block in the final partition. If there is a 
cluster $C'\in {\cal Z}$  of size more than $5$, then we run Lloyd's  algorithm on $(C',\lceil |C'|/5 \rceil)$ and refine the clustering of $C$. That is, the new clustering is obtained by replacing $C'$ with the clusters obtained in this run of Lloyd's  algorithm. We continue this process until all the clusters have size at most $5$. 
Thus we obtain  a partition $\{C_1,\ldots,C_{\ell}\}$ of $C$ of clusters of size at most $5$.  Now we partition $X$  into $X_1,\ldots,X_{\ell}$ as follows. 
For each $i\in \{1,\ldots,\ell\}$, we let $X_i$ be the set of vectors in $X$ such that for each vector $\bfx\in X_i$, the closest vector $c$ from $C$ to $\bfx$ is from $C_i$ (here, we break ties arbitrarily).  Let $\bfA_i$ be the matrix whose columns are the vectors of $X_i$. 
For each $i\in \{1,\ldots,\ell\}$, we run {\sf LRBMF} on $(\bfA_i,|C_i|)$; let $\bfB_i$ be the output. 
Since $\sum_{i=1}^{\ell}|C_i|=20$, the rank of the matrix resulted  by merging all $\bfB_i$s is at most $20$. 
The final output of our algorithm is obtained by merging the matrices $\bfB_1,\ldots,\bfB_{\ell}$. This completes the high level description of our algorithm for the case when $r=20$. The complete technical details of our algorithm is explained in the next section and experimental results of our algorithms are explained in the last section.

%% file: algo.tex
\section{Algorithms}

\label{algo}

We define a more general problem called \clustering{p}{q}, and prove that, in fact, \probgen{p}{q} is a particular case of \clustering{p}{q}. Before describing \clustering{p}{q}, let us introduce some notations.   
Recall that, for a number $q\geq 0$, a prime number $p>1$, and two vectors $\bfx,\bfy \in \{0,1,\ldots,p-1\}^m$, the distance between $\bfx$ and $\bfy$ in $\ell_q$ is $||\bfx-\bfy||_q=(\sum_{i=1}^m (\bfx[i]-\bfy[i])^q)^{1/q}$. Here, for notational convenience we use $0^0=0$. The differences $\bfx[i]-\bfy[i]$ of the vector coordinates are computed modulo $p$. The summation $\sum_{i=1}^m$ and multiplications are over the field of real numbers. 
For a number $q\geq 0$, a set of vectors $C$, and a vector $\bfx$, define $d_q(\bfx,C)=\min_{\bfc\in C} || \bfx-\bfc||^q_q$.  When $C=\{\bfc\}$,  we write $d_q(\bfx,\bfc)$ instead of $d_q(\bfx,C)$.

A $k$-ary relation $R$ over $\{0,\ldots,p-1\}$ is a set of $k$-tuples with elements from $\{0,\ldots,p-1\}$. A $k$-tuple $t=(t_1,\dots, t_k)$ \emph{satisfies} $R$ if $t$ is equal to one of the $k$-tuples from $R$.  
\begin{definition}[Vectors satisfying $\cR$]
Let $p>1$ be a prime number and let $\cR=\{R_1, \dots, R_m\}$ be a set of $k$-ary relations over $\{0,1,\ldots,p-1\}$. We say that a set $C=\{\bfc_1, \bfc_2, \dots, \bfc_k\}$ of $m$-dimensional vectors over GF$(p)$ \emph{satisfies $\cR$},  
if 
 $(\bfc_1[i],\ldots,\bfc_k[i])\in R_i$ for all $i\in \{1,\ldots,m\}$.
\end{definition}

Next, we formally define \clustering{p}{q}, where  $q\geq 0$ and $p>1$ is a  prime, and then prove that 
indeed  \probgen{p}{q} is a special case of \clustering{p}{q}. 

\defproblem{\clustering{p}{q}}{A set $X\subseteq \{0,1,\ldots,p-1\}^m$ of $n$ vectors, a positive integer $k$, and a set of $k$-ary relations
$\cR=\{R_1, \dots, R_m\}$. }{Among all  vector sets $C=\{\bfc_1,\ldots,\bfc_k\}\subseteq \{0,1,\ldots,p-1\}^m$ satisfying $\cR$, find a set $C$ minimizing the sum 
$\sum_{\bfx\in X} d_q(\bfx,C)$.}

The proof of the following lemma is almost identical to the proof of 
Proposition~\ref{prop:matrixFasRClust}, and hence omitted here. 
%

\begin{lemma}\label{lem:pqmatrixFasRClust} For any instance $(\bfA,r)$ of  \probgen{p}{q} 
 one can  construct in time $\OO(m+n+p^{2r})$ an instance $(X,k=p^r,\cR)$ of \clustering{p}{q} 
with the following property:
\begin{itemize} 
 \item for any solution $C$ of $(X,k, \cR)$,  there is an algorithm that in  time $\OO(p^{2r}m)$ returns   a solution $\bfB$ of $(\bfA,r)$ with the same cost as $C$, and 
 \item for any  solution $\bfB$ of $(\bfA,r)$,  there is an algorithm that in  time $\OO(p^{2r}m)$ returns   a  solution $C$ of $(X,k, \cR)$ with the same cost as $C$.
 \end{itemize}
\end{lemma}

Thus, to solve the low-rank matrix factorization problem over a finite field GF($p$), it is enough to design an algorithm for 
\clustering{p}{q}.   
Let $I=(X,k,{\mathcal R}=\{R_1,\ldots,R_m\})$ be an instance of \clustering{p}{q}   and let $C=\{\bfc_1,\ldots,\bfc_k\}$ be a solution to $I$. 
We call $C$ to be the set of {\em centers}. Then, define the cost of  the solution $C$ of the instance $I$ to be $\Hdist(X,C)=\sum_{\bfx\in X} d_q(\bfx,C)$. 
Also, given the set $C$, there is a natural way one can partition the set of vectors $X$ into $k$ parts $X_1\uplus\cdots\uplus X_k$ as follows. 
For each vector $\bfx$, let $i$ be the smallest index such that  $\bfc_i$ is a closest vector to $\bfx$ from $C$.  
 Then, $\bfx \in X_i$. 
 This implies that   

\begin{equation}
\label{clufrctr}
\Hdist(X,C)=\sum_{i=1}^{k}\sum_{\bfx\in X_i} d_q(\bfx,\bfc_i)
\end{equation}
%
We call  such partition \emph{clustering of $X$ induced by $C$} and  the sets $X_1,\ldots,X_k$ as the  {\em clusters corresponding to~$C$}. 

Next, we explain how we compute the best possible centers from a given set of clusters of $X$. For a  partition $X_1\uplus\cdots\uplus X_k$ of $X$, $i\in [m]$, and $(b_1,\ldots,b_k)\in R_i$, define 
\begin{equation}
\label{pq_centerfromclusters}
g_i(b_1,\ldots,b_k)=\sum_{j\in [k]}\sum_{\bfx\in X_j}|\bfx[i]-b_j|^q
\end{equation}

Let the set $\{\bfc_1,\ldots,\bfc_k\}$ be such that for any $i\in [m]$, $(\bfc_1[i],\ldots,\bfc_k[i])=\argmin_{b\in R_i} g_i(b)$. One can easily verify that $\{\bfc_1,\ldots,\bfc_k\}$ is a best possible set of centers for the clusters $X_1,\ldots,X_k$.

Our algorithm \algoclus{p}{q} for \clustering{p}{q} has the following steps. 

\begin{itemize}
\setlength{\itemindent}{2em}
\item[Step 0:] Set $minCost : =\infty$ and $k=p^r$. 
\item[Step 1:] Let $X_{1}\uplus\cdots\uplus X_{k}$ be a random partition of $X$. 
\item[Step 2:]
\label{stepcont} 
 Using \eqref{pq_centerfromclusters},  compute a solution $C$ from the partition $X_{1}\uplus\cdots\uplus X_{k}$.  
\item[Step 3:] Find clusters $Y_1,\ldots,Y_{k}$ corresponding to $C$ (i.e., $C$ and $\{Y_{1},\ldots,Y_{k}\}$ satisfies \eqref{clufrctr}).
\item[Step 4:] If $\Hdist(X,C)= minCost$, then  output $C$ and   stop. Otherwise, set $minCost=\Hdist(X,C)$, and $X_i:=Y_i$ for all $i\in [k]$. Then, go to Step 2. 
\end{itemize}

Notice that when $q=1$, the maximum error can be $pmn$. Thus the number of iterations in \algoclus{p}{1} is at most $pmn$ and each iteration takes time $\OO(p^r(m+n))$. Thus, the worst case running time of \algoclus{p}{1} is $\OO(p^{r+1}(m+n)mn)$.


%
%
%

\paragraph*{Algorithm for \probgen{p}{q}}
Recall that \probgen{p}{q} is a special case \clustering{p}{q} (see Lemma~\ref{lem:pqmatrixFasRClust}). 
For a given instance $(\bfA,r)$ of \probgen{p}{q}, we apply Lemma~\ref{lem:pqmatrixFasRClust} and construct 
an instance $(X,k=p^r,\RR)$ of \clustering{p}{q}. Then, we run \algoclus{p}{q} on $(X,k=2^r,\RR)$ 10 times and take the best output among these 10 executions.  In the next section we explain about the experimental evaluations of the algorithm for \probgen{p}{q}. We call our algorithm for \probgen{p}{1} as {\sf LRMF($p$)}.

\paragraph*{Algorithm for \bmfbr}
We have mentioned that \clustering{p}{q} is general problem subsuming \rclustering\ and \bmfbr\ is a special case of \rclustering. Next, we explain, how to obtain an equivalent instance of \clustering{2}{1} 
from a given  instance $(\bfA,r)$ of \bmfbr. Towards that apply Proposition~\ref{prop:matrixFasRClust}, and get 
an instance $(X,k=2^r,\RR)$ of \rclustering from the instance  $(\bfA,r)$ of \bmfbr. In fact, this instance 
$(X,k=2^r,\RR)$ is the required instance of \clustering{2}{1}. 
Next, we run \algoclus{2}{1} on $(X,k=2^r,\RR)$ 10 times and take the best output among these 10 executions.  
We call this algorithm as {\sf LRBMF}.

%

\begin{algorithm}[tb]
   \caption{{\sf PLRBMF}}
   \label{algpar}
\begin{algorithmic}[1]
\STATE Let $X$ be the set of columns of $\bfA$
\STATE $Cost:=0$
\STATE $d':= r_s \cdot d$\;
\STATE Run Lloyd's $k$-means clustering algorithm on $(X,r)$. Let $X_1,\ldots, X_{d'}$ be the output clusters and let $\bfz_1,\ldots,\bfz_{d'}\in {\mathbb R}^m$ be the output cluster centers \label{firststep}
\STATE Run Lloyd's $k$-means clustering algorithm on $(\{\bfz_i \colon i\in [d']\},d)$. Let ${\cal Z}$ be the set of output clusters
\WHILE{there exists $Z\in {\cal Z}$ such that $\vert Z\vert > r_s$}
\STATE Run Lloyd's $k$-means clustering algorithm on $(Z,\lceil \frac{\vert Z\vert }{r_s} \rceil)$. Let ${\cal Z}'$ be the set of output clusters
\STATE ${\cal Z}:=({\cal Z}\setminus \{Z\})\cup {\cal Z}'$\;
\ENDWHILE
\STATE For each $Z\in {\cal Z}$, let $X_Z$ be the union of the clusters from   $\{X_1,\ldots, X_{d'}\}$ such that the corresponding cluster centers (see Line~\ref{firststep}) belongs to $Z$. Notice that $\{X_Z\colon Z\in {\cal Z}\}$ is a partition of $X$
\STATE For each $Z$, run {\sf LRBMF} on $([X_Z],\vert Z\vert)$ and let $M_Z$ be the output. Here, $[X_Z]$ is the matrix where the set of columns is $X_Z$
\STATE Let $D$ be the union of the set of columns of the matrices in $\{M_Z~\colon~ Z\in {\cal Z}\}$
\STATE The output matrix $M$ is constructed as follows. For each $i\in [n]$, the $i$th column of $M$ is the vector in $D$ which is closest to the $i$th column of $\bfA$. 
\end{algorithmic}
\end{algorithm}

\paragraph*{Algorithms for large rank}
Notice that the running time of \algoclus{p}{q} is at least $p^r$. Thus, to get a fast algorithm for large $r$ we propose the following algorithm (call it \algolargeclus{p}{q}). Thus the running times of {\sf LRMF($p$)} and {\sf LRBMF} are at least $2^{r}$.  For large $r$, instead of running {\sf LRBMF} (or {\sf LRMF($p$)}) we partition the columns of the input matrix into blocks and we run  {\sf LRBMF} (or {\sf LRMF($p$)}) on each of these blocks with for rank at most $r_s$ such that the sum of the rank parameters among the blocks is at most $r$. Then, we merge the outputs of each of these small blocks. We call these new algorithms {\sf PLRBMF} and {\sf PLRMF($p$)}.

The input of {\sf PLRBMF} is an instance of \bmfbr\ and two integers $r_s$ and $d$ such that $r_s\cdot d \leq r$, where $r$ is the rank of the output matrix. Similarly, the input of {\sf PLRMF($p$)} is an instance of  \probgen{p}{1} and two integers $r_s$ and $d$ such that $r_s\cdot d \leq r$, where $r$ is the rank of the output matrix. That is, here we specify $r_s$ and $d$ as part of input and we want our algorithms to use 
{\sf LRMF($p$)} or {\sf LRBMF} with rank parameter at most $r_s$ and finally construct an output of rank at most $r_s\cdot d$. That is, given $r$, one should choose $d$ to be the largest integer such that $r_s\cdot d \leq r$, where $r_s$ is the largest rank that is practically feasible for running {\sf LRMF($p$)} and {\sf LRBMF} 
for the input matrices we consider. 

Here, we explain the algorithm  {\sf PLRBMF}. The steps of the algorithm  {\sf PLRMF($p$)} are identical to {\sf PLRBMF} and hence we omitted those details. The pseudocode of the algorithm {\sf PLRBMF} is given in Algorithm~\ref{algpar}. The input for {\sf PLRBMF}  is $(\bfA,r,r_s,d)$, where $r_s\cdot d \leq r$. We would like to remark that when $d=1$, {\sf PLRBMF} is same as {\sf LRBMF}. 

Next we analyze the running time.  
The algorithm  {\sf PLRBMF} calls Lloyd's $k$-means algorithm at most $1+r$ times. As the maximum error is at most $mn$, the total number of iterations of  Lloyd's algorithm in all executions together is $(1+r)mn$. Moreover each iteration takes $\OO(rmn)$ time. At the end we run at most $r$ iterations of {\sf LRBMF} with rank being $r_s$. Thus the total running time is $\OO(r^2m^2n^2+2^{r_s}(m+n)mn)$.

%% file: exp.tex
\section{Experimental Results}
\label{expres}


\begin{table}[htp]
    \caption{Comparison on synthetic data.  
     The entries in the table are average error and standard deviations on 10 random $50\times 100$ matrices. 
     Here the ranks of the output matrices are $\{1,\ldots,5\}\cup \{10,15,20,25,30\}$. Standard deviations are mentioned in brackets.}
    \label{table1}
    \centering
    \resizebox{\columnwidth}{!}{
      \begin{tabular}{|c|c|c|c|c|c|c|}
    \specialrule{.2em}{.1em}{.1em}
    Rank     & 1     & 2     & 3     & 4     & 5              \\
    \specialrule{.2em}{.1em}{.1em}
    {\sf PLRMF($2$)} &2143.6	&1922.5	&1772.1	&1657.8	&1552.6 \\

         &(13.9)	&(12.5)	& (18.6)	& (12.2)	& (12.7) \\
    \hline

    {\sf PLRBMF} &2143.9	&1946.8	&1823.1	&1723.6	&1646.1 \\

       & (15.1)	& (8.2)	& (13.9)	& (14.5)	& (9.9) \\
    \hline

    {\sf BMFZ} &2376.5 &2204.6 &2106.7 &2023.5 &1941.2 \\

       & (14.7)	& (18.3)	& (19)	& (13.7)	& (14.7)\\
    \hline

    {\sf NMF} &2424.8 &2303.1 &2205.4 &2114.0 &2041.0 \\

        &(4)	& (6)	& (5)	& (11.4)	& (9.4) \\
    \hline

    {\sf ASSO} &2481.5 &2447.5 &2414.9 &2383.2 &2352.3 \\
   
      &(43.1)	&(42.4)	&(42.3)	&(41.9)	&(41.2) \\
 
    \specialrule{.2em}{.1em}{.1em}
	\end{tabular}
	}

%
%
    \centering
    \resizebox{\columnwidth}{!}{
    \begin{tabular}{|c|c|c|c|c|c|c|}
    \specialrule{.2em}{.1em}{.1em}
    Rank     & 10     & 15     & 20     & 25     & 30              \\
    \specialrule{.2em}{.1em}{.1em}
    {\sf PLRMF($2$)} &1374.1	&1190.2 &992	&818.6	&642.7\\
    
        &(18.4)	&(14.5)	&(10.7)	&(13.9)	&(19.2) \\
    \hline

    {\sf PLRBMF} &1412.5	&1221.8	&1067 &898.2 &776.4 \\
    
      &(16.2)	&(13)	&(20.2)	&(14.9)	&(36.4) \\
    \hline

    {\sf BMFZ} &1647.2	&1403.1 	& 1184.5	&972.6	&768.6 \\
    
     &(19.7)	&(19.5)	&(18.2)	&(13.8)	&(12.7) \\
    \hline

    {\sf NMF} &1780.7	&1600.4	&1460.7	&1337.5	&1214.7 \\
    
     &(11.8)	&(9.5)	&(14.2)	&(10.5)	&(21.6) \\
    \hline

    {\sf ASSO} &2201.8	&2055.7	&1913.1	&1773.9	&1637.7 \\
    
        &(38)	&(36.3)	&(34.2)	&(32.3)	&(31.3) \\

    \specialrule{.2em}{.1em}{.1em}
	\end{tabular}
	}
\end{table}

We analyze our algorithm for \probgen{p}{1} (called {\sf PLRMF($p$)}), and  
\bmfbr (called {\sf PLRBMF}) on synthetic data and real-world data.  We use the $r_s$ value to be $5$ 
for {\sf PLRBMF} and {\sf PLRMF($2$)}. 
That is,  {\sf PLRBMF} is same as {\sf LRBMF}  and {\sf PLRMF($2$)} is same as {\sf LRMF($2$)}
and when $r\leq 5$. 
We run all the codes in a laptop with specification Intel Core i5-7200U CPU, $2.50$GHz $\times 4$, and 8GB RAM. We compare our algorithms with   the following algorithms.

\begin{itemize}

\item \texttt{Asso} is an algorithm for \bmfbr 
by Miettinen et al.~\cite{MiettinenMGDM08}.

\item One of the closely related problem is Non-negative Matrix Factorization (NMF), where we are given a matrix 
$\bfA\in {\mathbb R}^{m\times n}$ and an integer $r$, and the objective is to find two factor matrices $\bfU \in {\mathbb R}^{m\times r}$ and  $\bfV \in {\mathbb R}^{r \times n}$ with non-negative entries such that the squared  Frobenius norm of  $\bfA-\bfU \bfV$ is minimized. We compare our algorithms 
with 
the algorithms  for  NMF (denoted by  {\sf NMF}) designed in~\cite{lee99}. 
 We used the implementation from \url{https://github.com/cthurau/pymf/blob/master/pymf/nmf.py}.
  The details about error comparisons are different for synthetic and real-world data and it is explained in the corresponding subsections.  

\item 
Recall that Kumar et al.~\cite{Kumar19f} considered the following problem. Given a binary matrix $\bfA$ of order $m\times n$ and an integer $r$, compute two binary matrices $\bfU \in \{0,1\}^{m\times r}$ and $\bfV\in \{0,1\}^{r\times n}$ such that $||\bfU\cdot \bfV -\bfA||_F^2$ is minimized where $\cdot$ is the matrix multiplication over ${\mathbb R}$.   
Their algorithm is a two step process. In the first step they run the $k$-Means algorithm with 
the input being the set of rows of the input matrix  and the number of clusters being $2^r$ over reals. 
Then each row is replaced with a row from the same cluster which is closest to the center. Then in the second step a factorization for the the output matrix of step 1 (which has at most $2^r$ distinct rows) is obtained. For the experimental evaluation  Kumar et al. implemented the first step of the algorithm with number of centers being $r$ instead of $2^r$. We call this algorithm as {\sf BMFZ}. That is, here we get a binary matrix $\bfB$ with at most $r$ distinct rows as the  output. The error of our algorithm will be compared with $||\bfA-\bfB||_1$.
\end{itemize}

\begin{figure}[t]
\centering
\begin{subfigure}{}
\centering
\includegraphics[width = 7cm]{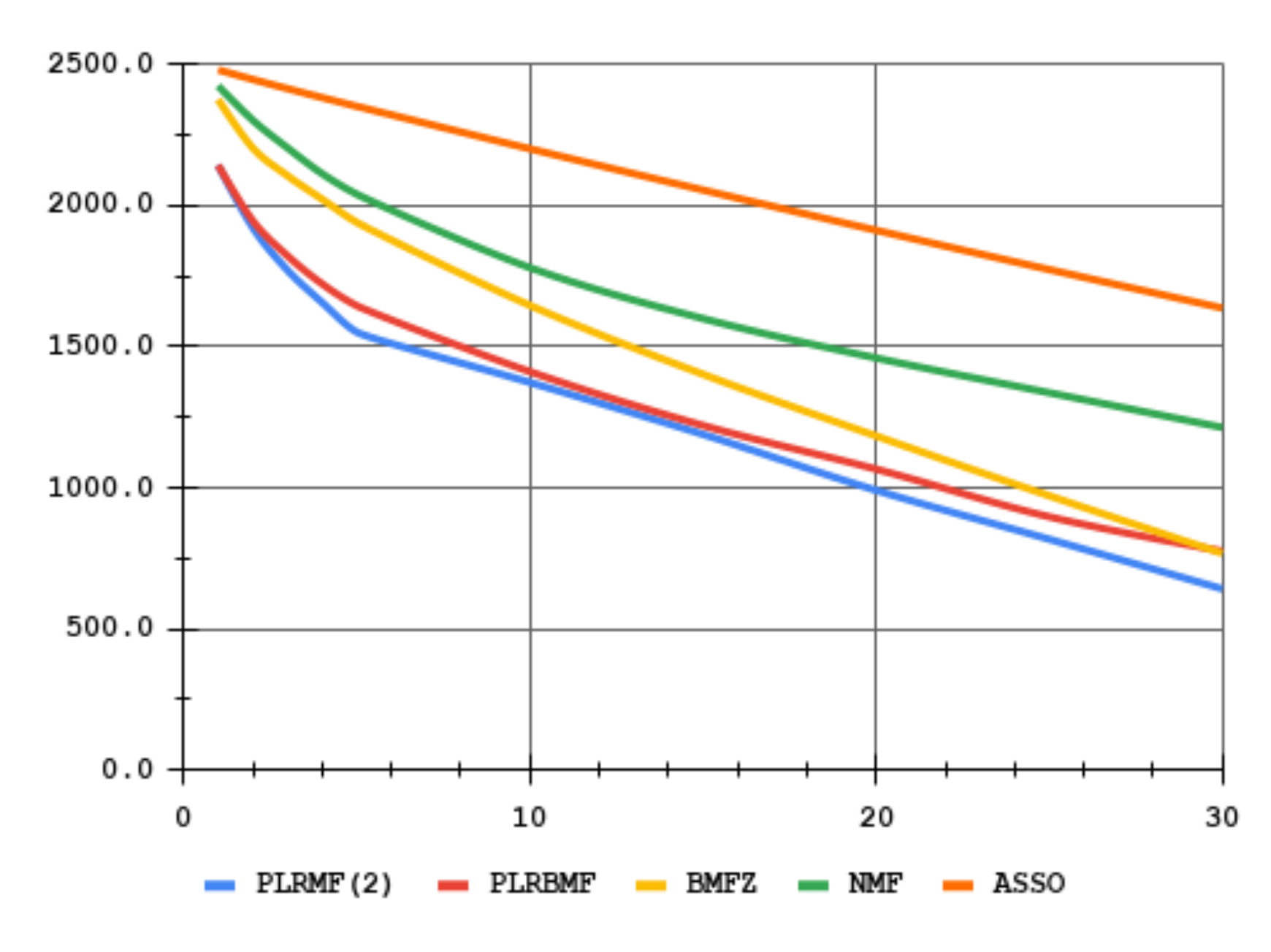}
\end{subfigure}
\begin{subfigure}{}
\includegraphics[width = 7cm]{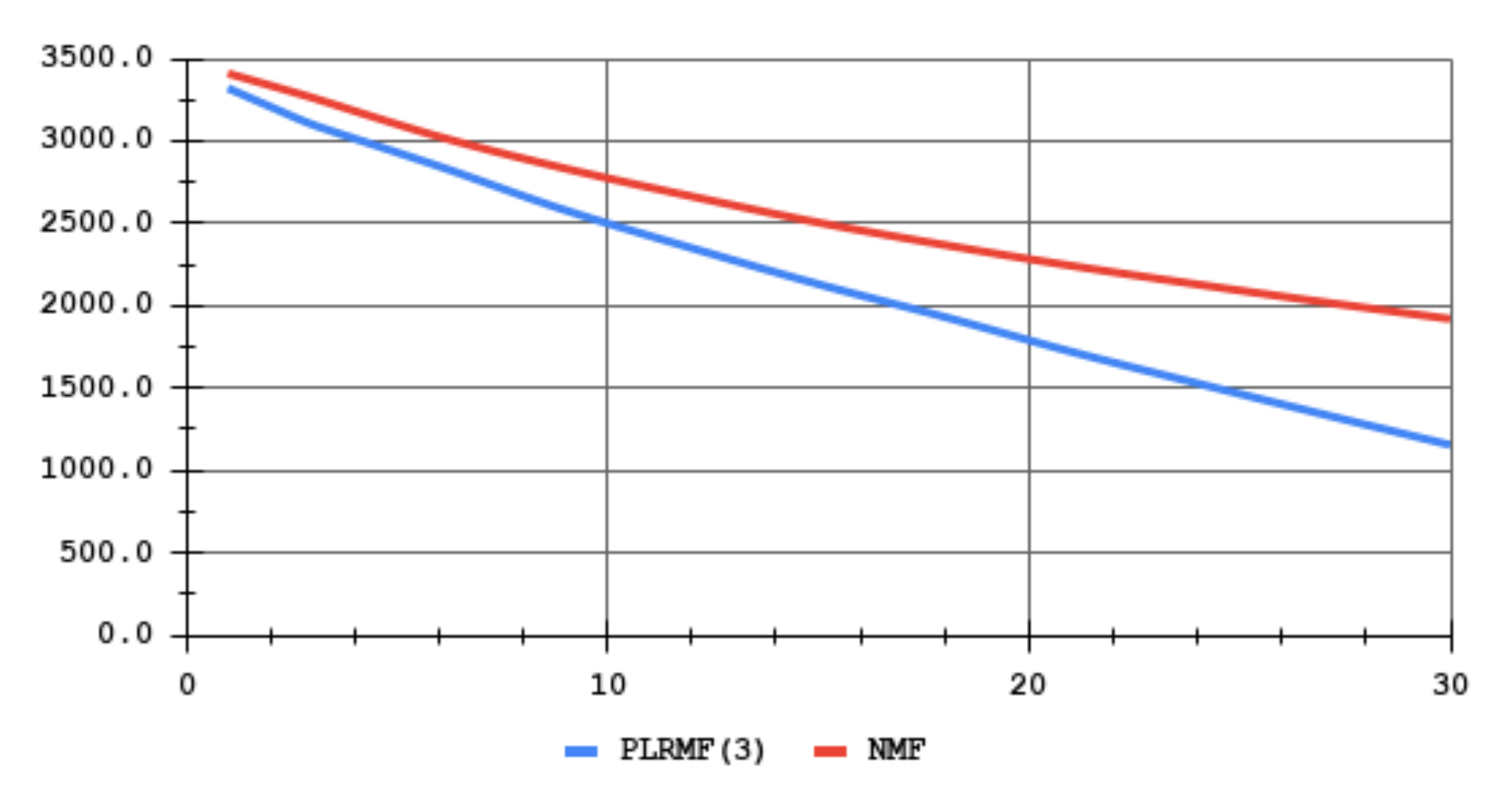}
\end{subfigure}%
\begin{subfigure}{}
\includegraphics[width = 7cm]{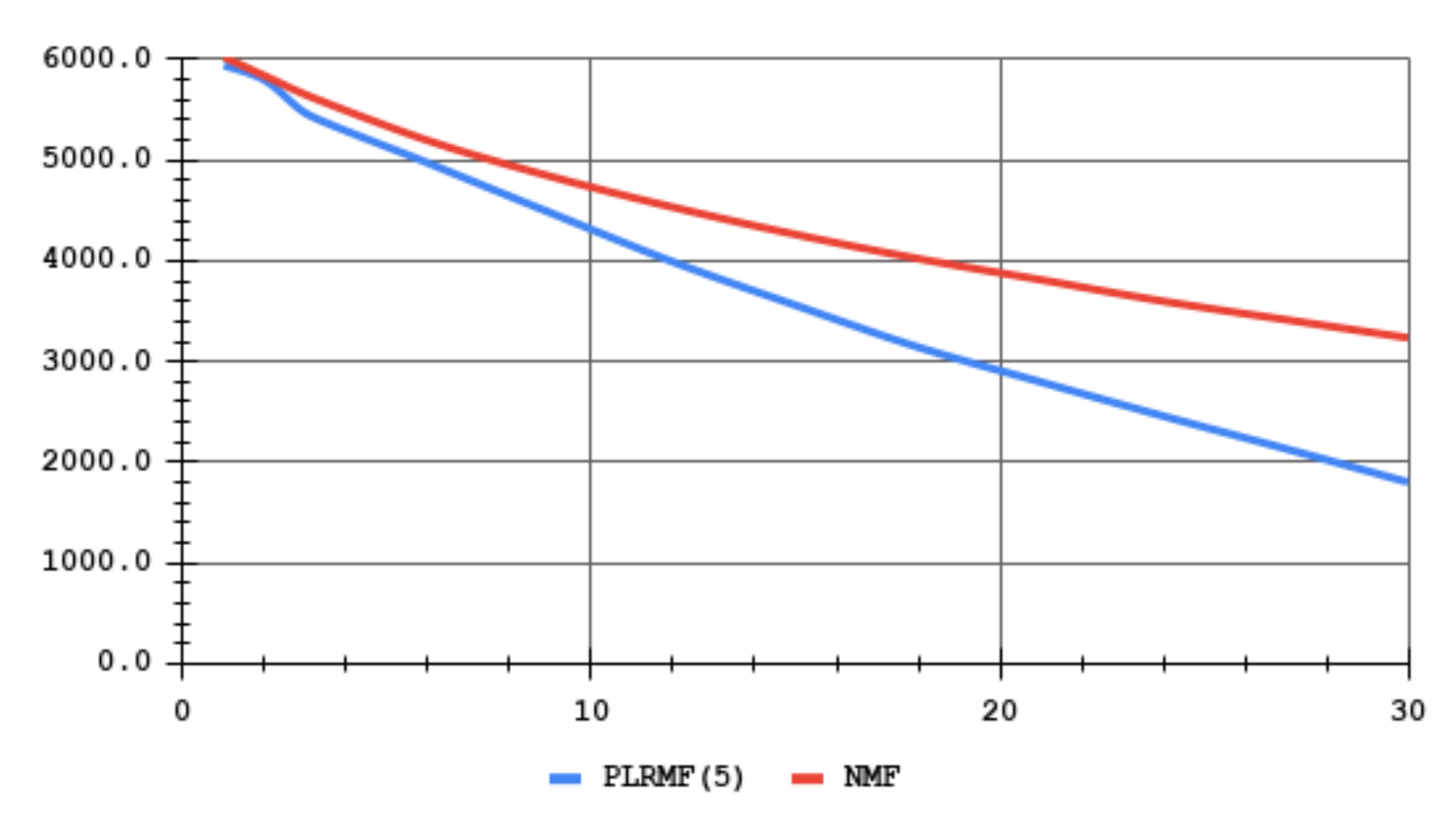}
\end{subfigure}
\caption{Graph on synthetic data where the entries in the table are average error on 10 random $50\times 100$ matrices.}
\label{fig:synthgraphs}
\end{figure}

\subsection{Synthetic Data}

%
%
%
We analyze our algorithms on binary matrices and compare with {\sf NMF}, {\sf BMFZ}, and \texttt{Asso} on random matrices of dimension $50\times 100$.  We run all the algorithms 10 times and take the best results. 
The output of the {\sf NMF} will be two factor matrices over reals.  
We compare the error of our algorithm with 
$||\bfA-\bfU \bfV||_1$,  where $\bfU$ and $\bfV$ are the factors output by {\sf NMF}  on the input~$\bfA$. 
The results are summarized in Tables~\ref{table1}. 
Even without rounding the factors of the output of {\sf NMF}, our algorithms perform better.  
We would like to mention that {\sf NMF} is designed to get factors with the objective of minimizing 
$||\bfA-\bfU \bfV||^2_F$. For our problem the error is measured in terms of $\ell_1$-norm and so we are getting better results than {\sf NMF}. 
We also compare {\sf PLRMF($3$)} and {\sf PLRMF($5$)} with {\sf NMF}. The performance of our algorithms are summarized in Figure~\ref{fig:synthgraphs}. 
{\sf PLRMF(2)} is giving $>15\%$ improvement over {\sf BMFZ} for  rank  $3$ to  $30$. 
{\sf PLRMF(5)} percentage improvement over {\sf NMF} is monotonously increasing: we have 
more than $3\%$ improvement on rank $3$, more than $11\%$ improvement on rank $12$, and more than $26\%$ improvement on rank~$21$.

%% file: expsup.tex
\subsection{Experimental Results on Real-world Data}
\label{expres}

    \begin{table}[t]
    
            \caption{Performance of our algorithm {\sf PLRBMF} compared to {\sf NMF} and {\sf BMFZ}. The dimension of the image is $561\times 800$.}
        \label{tbl:zebra}

        \centering
        
         \resizebox{\columnwidth}{!}{
        \begin{tabular}{|c|c|c|c|}
           \toprule
           &  {\sf NMF} & {\sf BMFZ} &{\sf PLRBMF}  \\
          \midrule
    \rotatebox{90}{Original image}&          \includegraphics[scale=0.08]{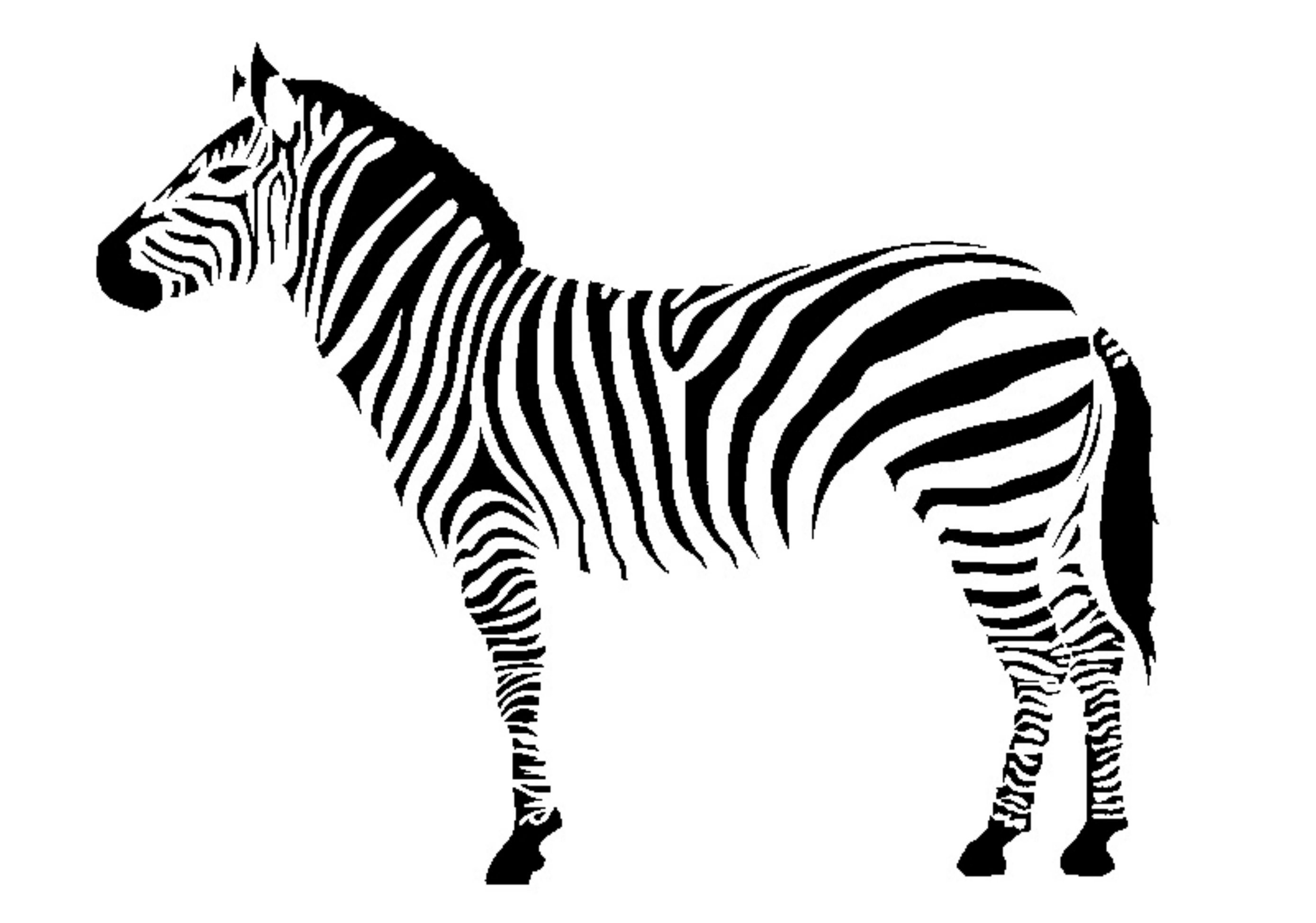} & \includegraphics[scale=0.08]{original_image} & \includegraphics[scale=0.08]{original_image} \\
          \midrule

   \rotatebox{90}{Rank: 10 }&        \includegraphics[scale=0.08]{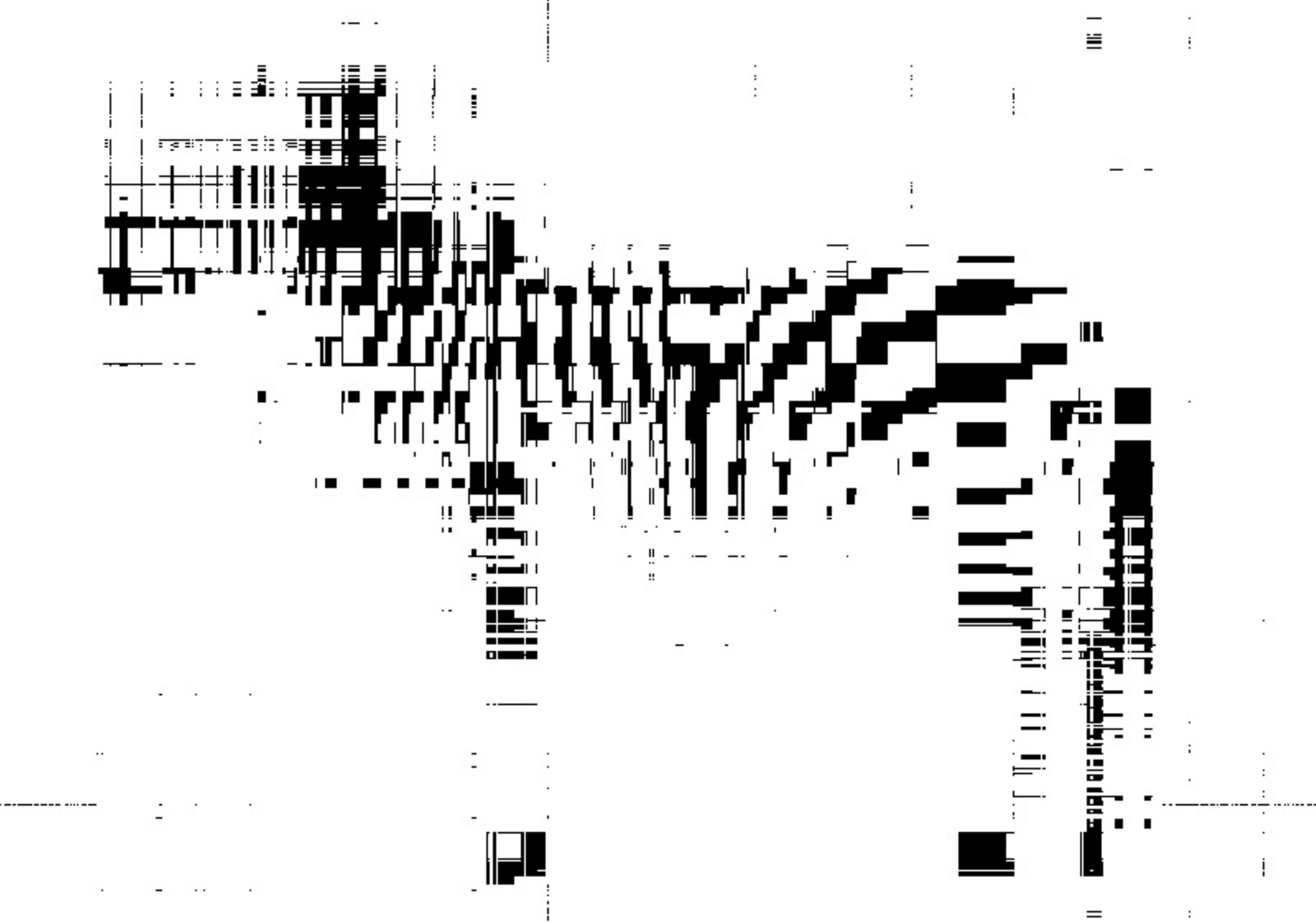} & \includegraphics[scale=0.08]{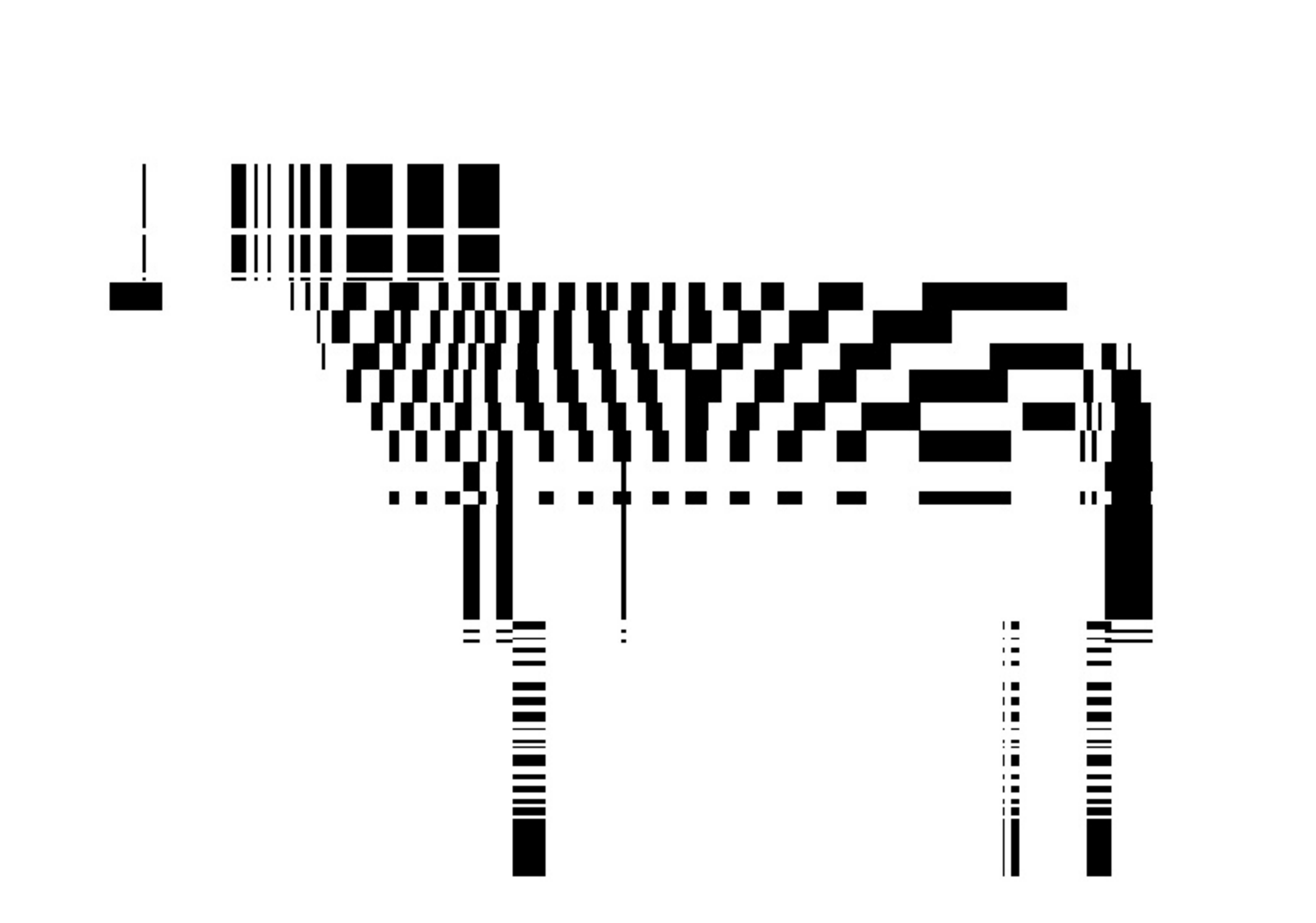} & \includegraphics[scale=0.08]{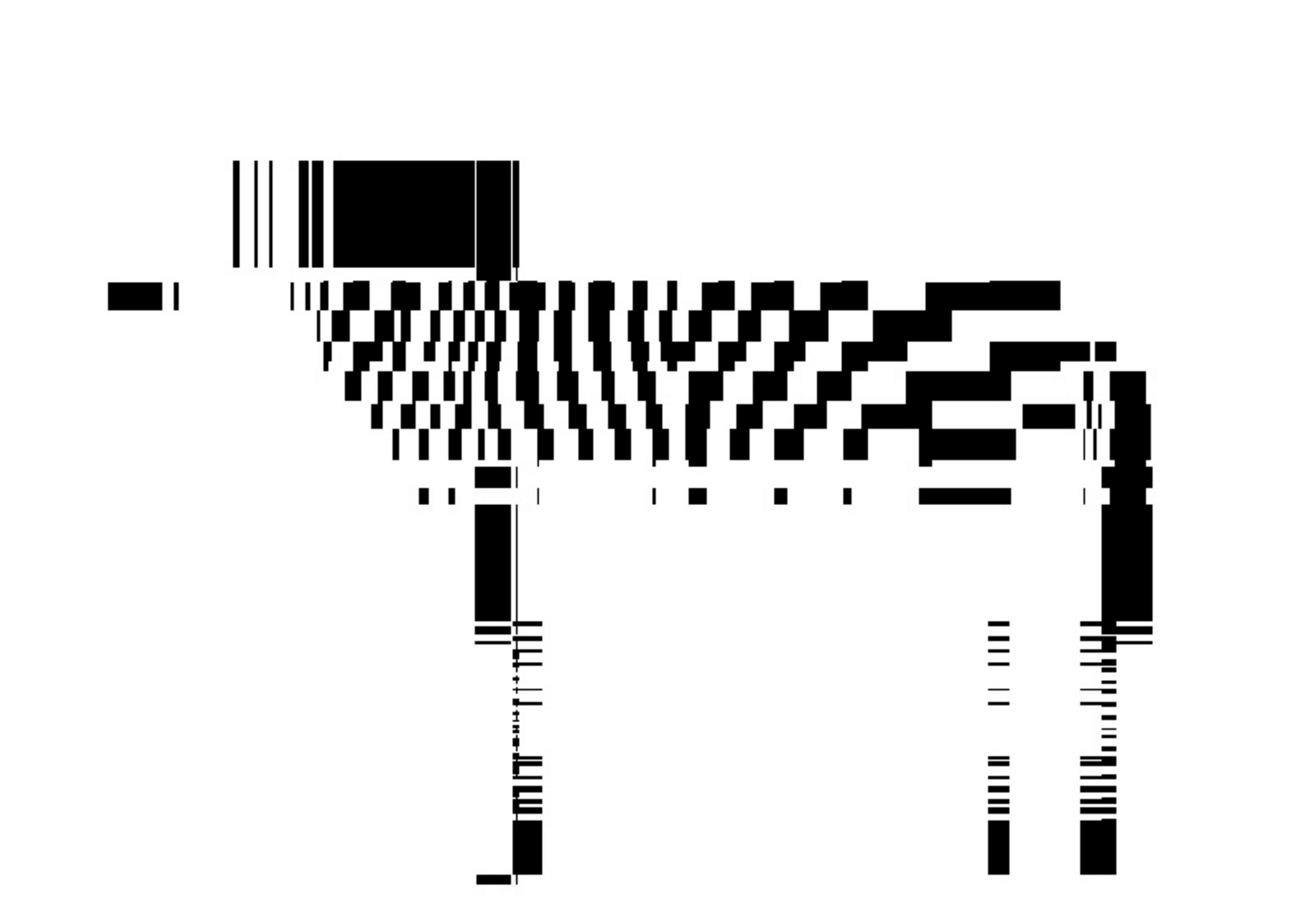} \\
& Error: 40213 & Error: 37485 & Error: 35034\\


          \midrule

\rotatebox{90}{Rank: 20 }&           \includegraphics[scale=0.08]{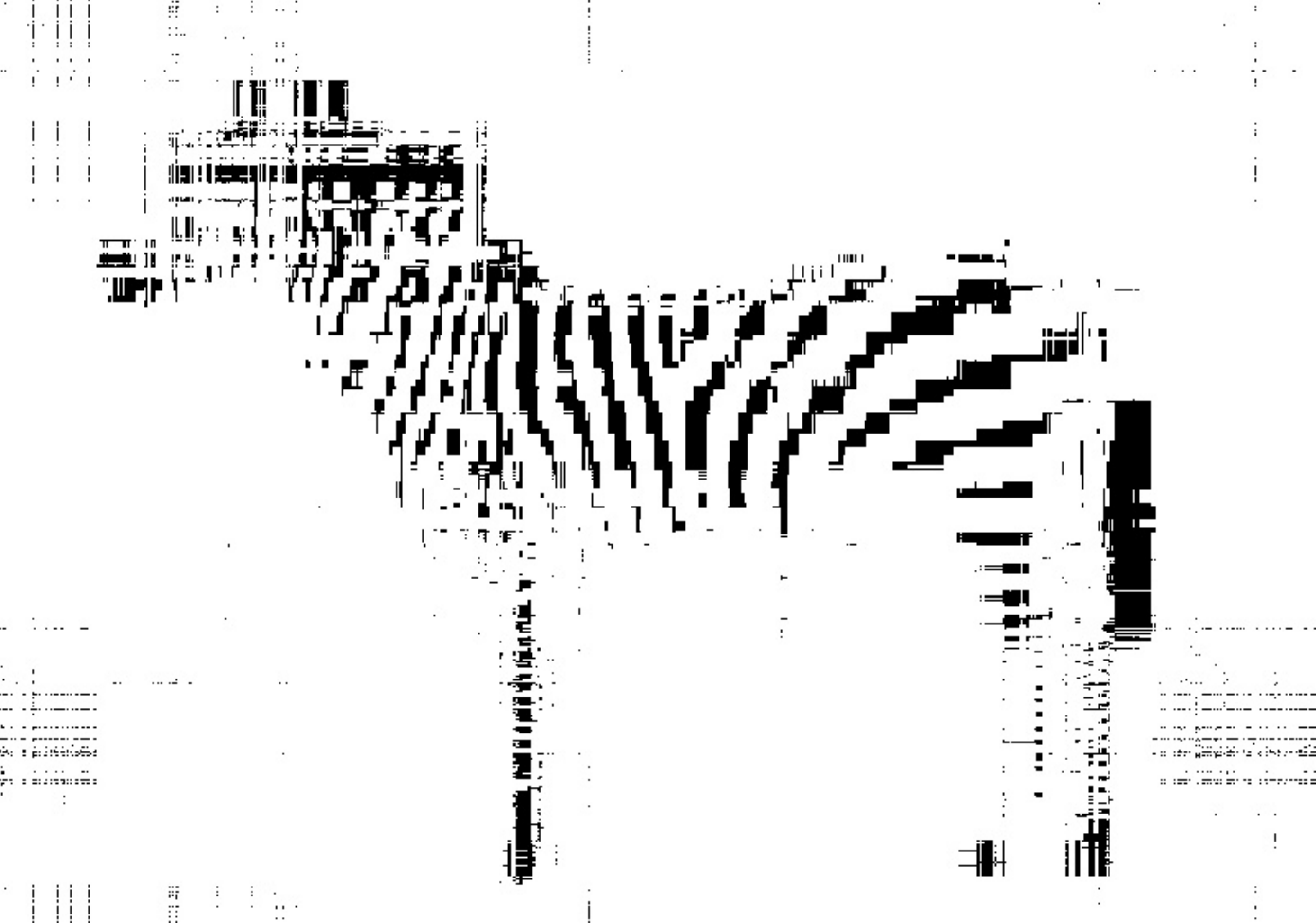} & \includegraphics[scale=0.08]{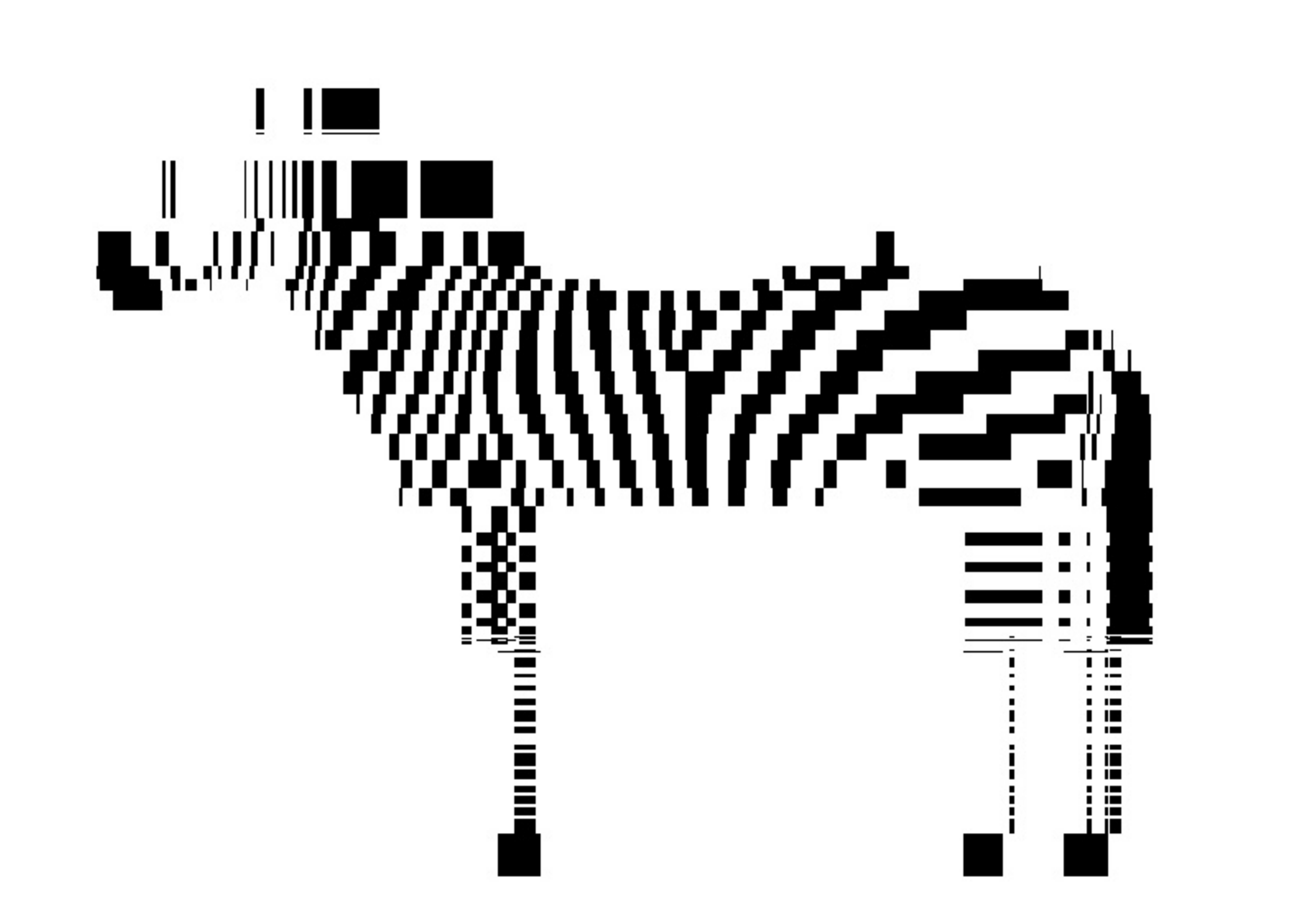} & \includegraphics[scale=0.08]{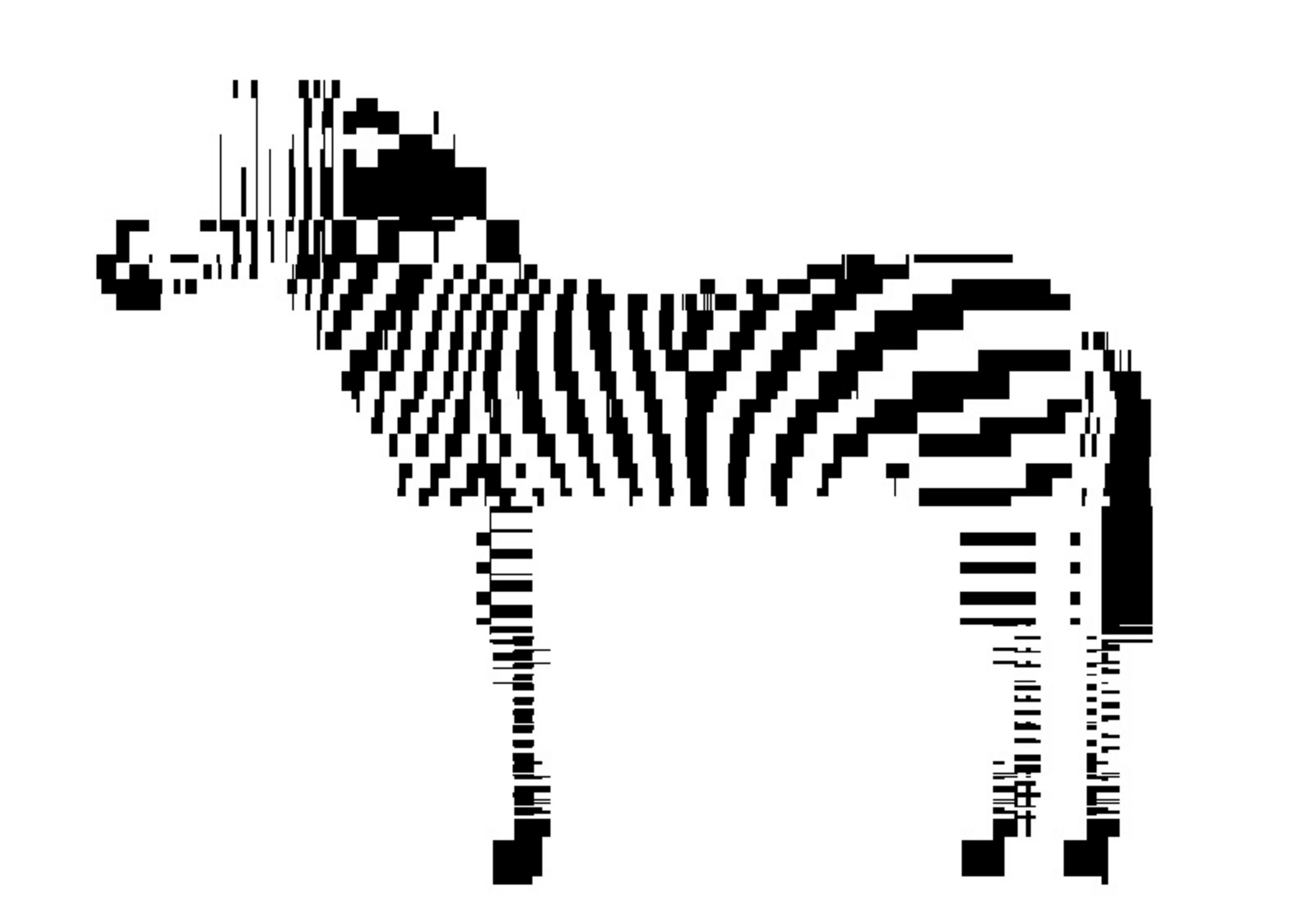} \\
& Error: 37288  & Error: 27180 & Error: 23763\\

          \midrule
\rotatebox{90}{Rank: 30 }&        \includegraphics[scale=0.08]{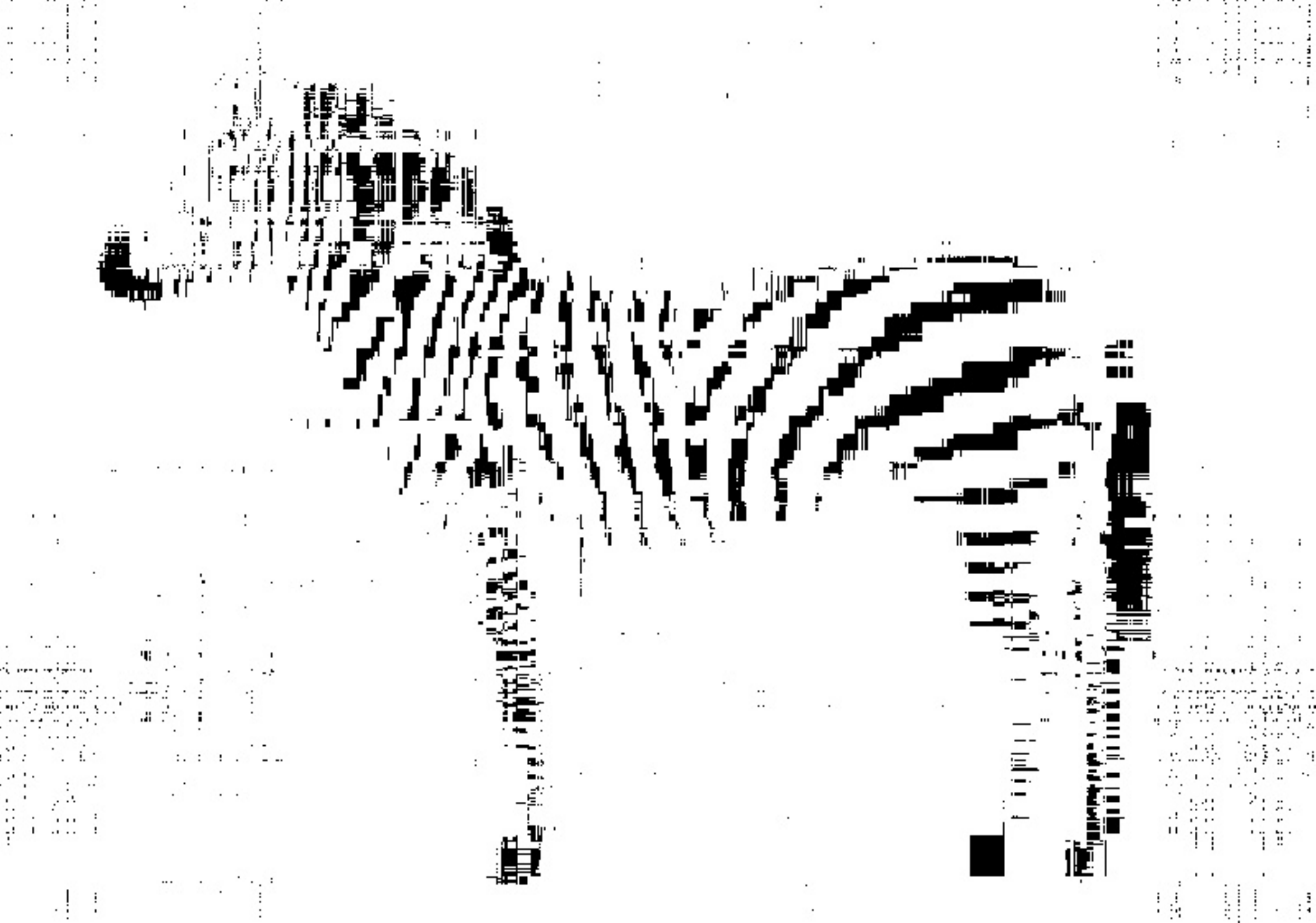} & \includegraphics[scale=0.08]{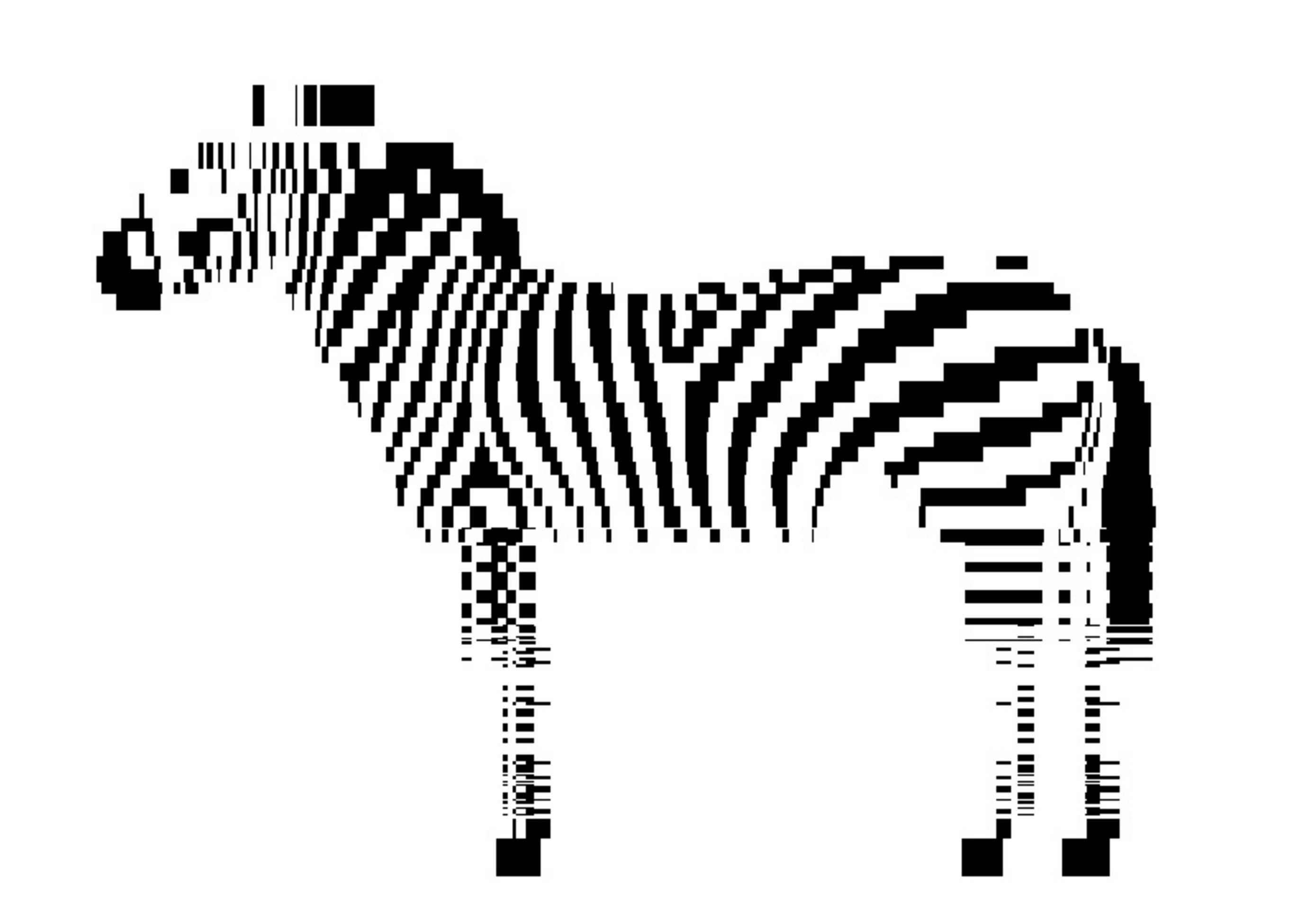} & \includegraphics[scale=0.08]{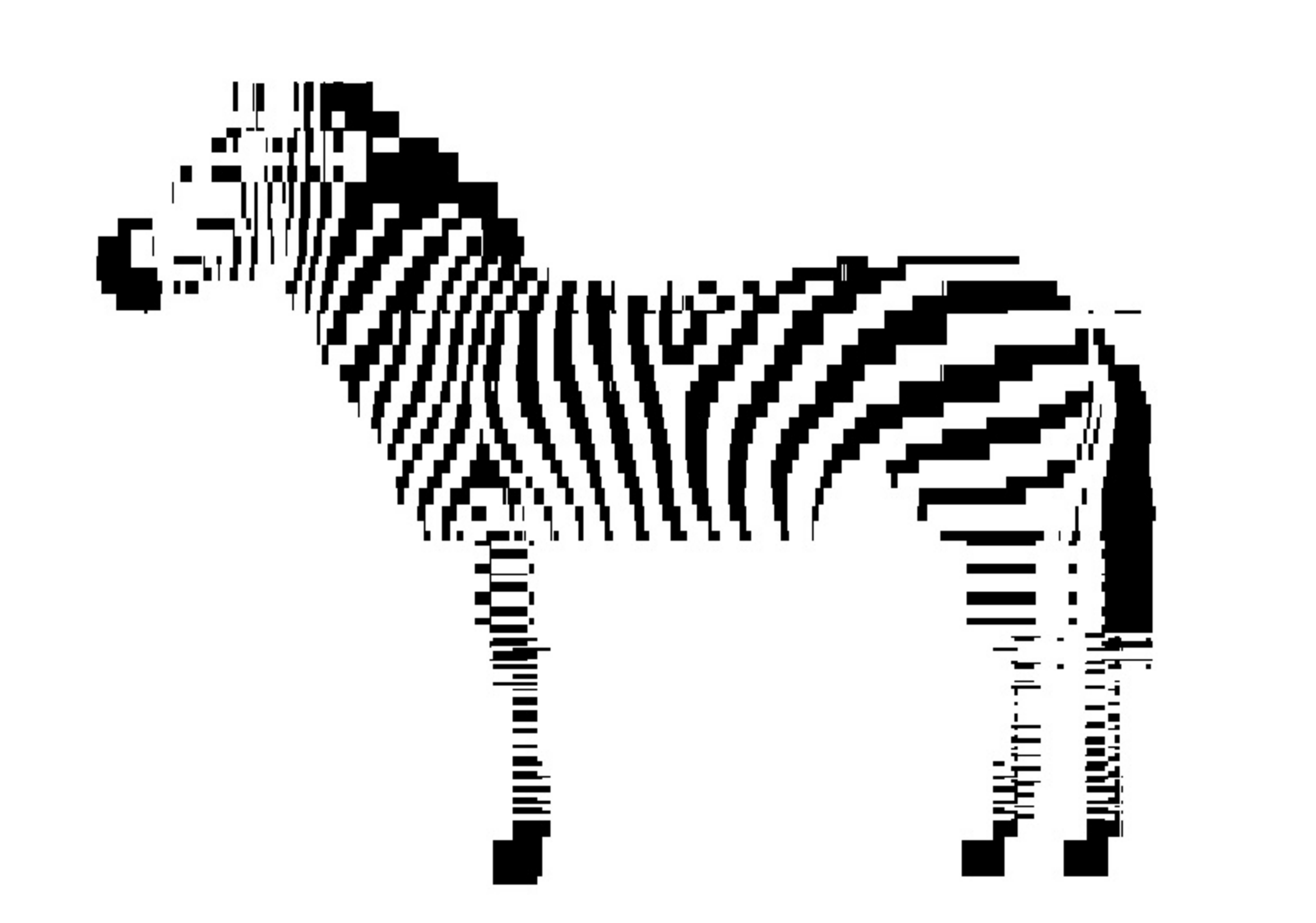} \\
& Error: 35938 & Error: 21115 & Error: 19081\\


          \midrule
          \rotatebox{90}{Rank: 40 }&  
              \includegraphics[scale=0.08]{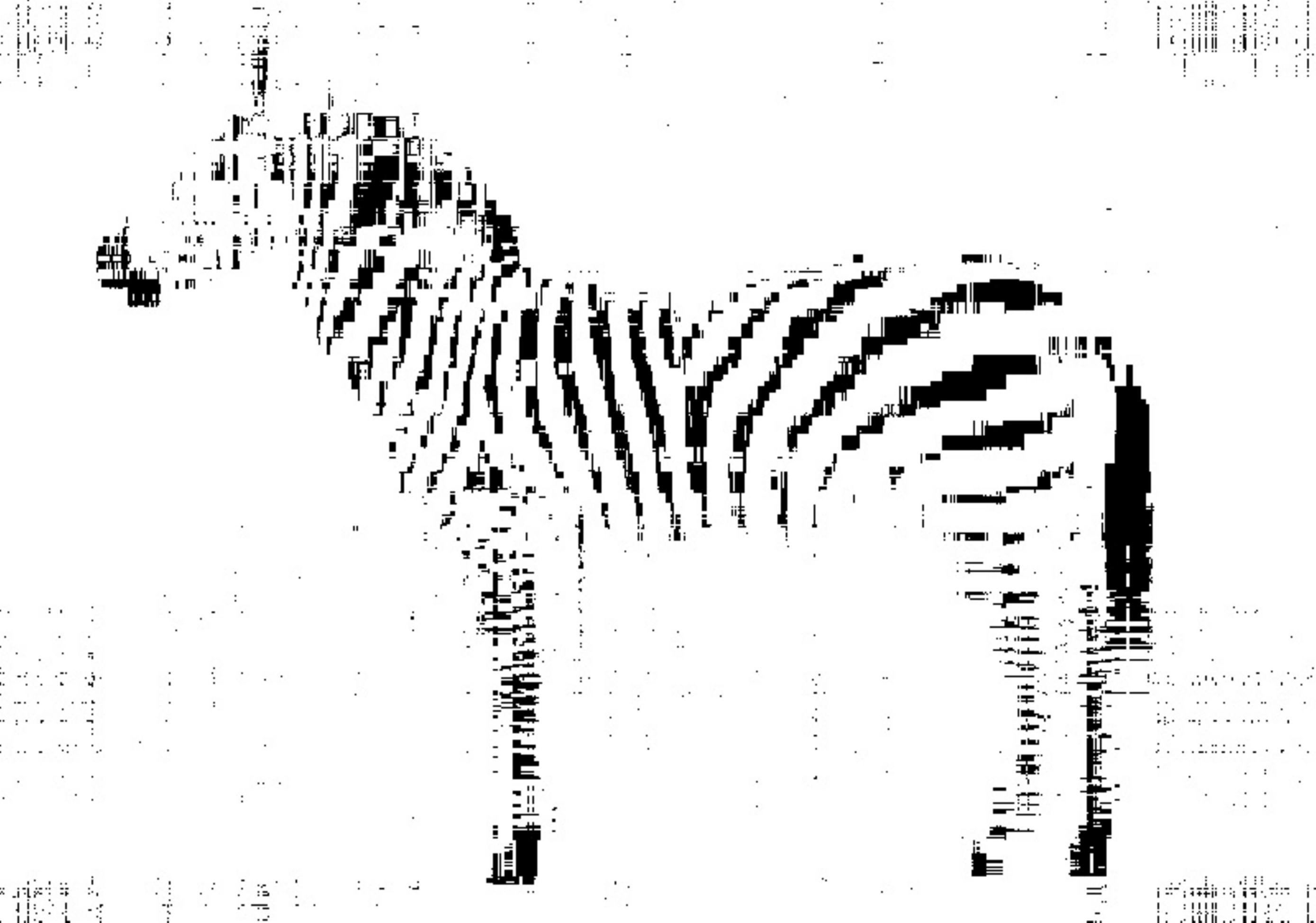} & \includegraphics[scale=0.08]{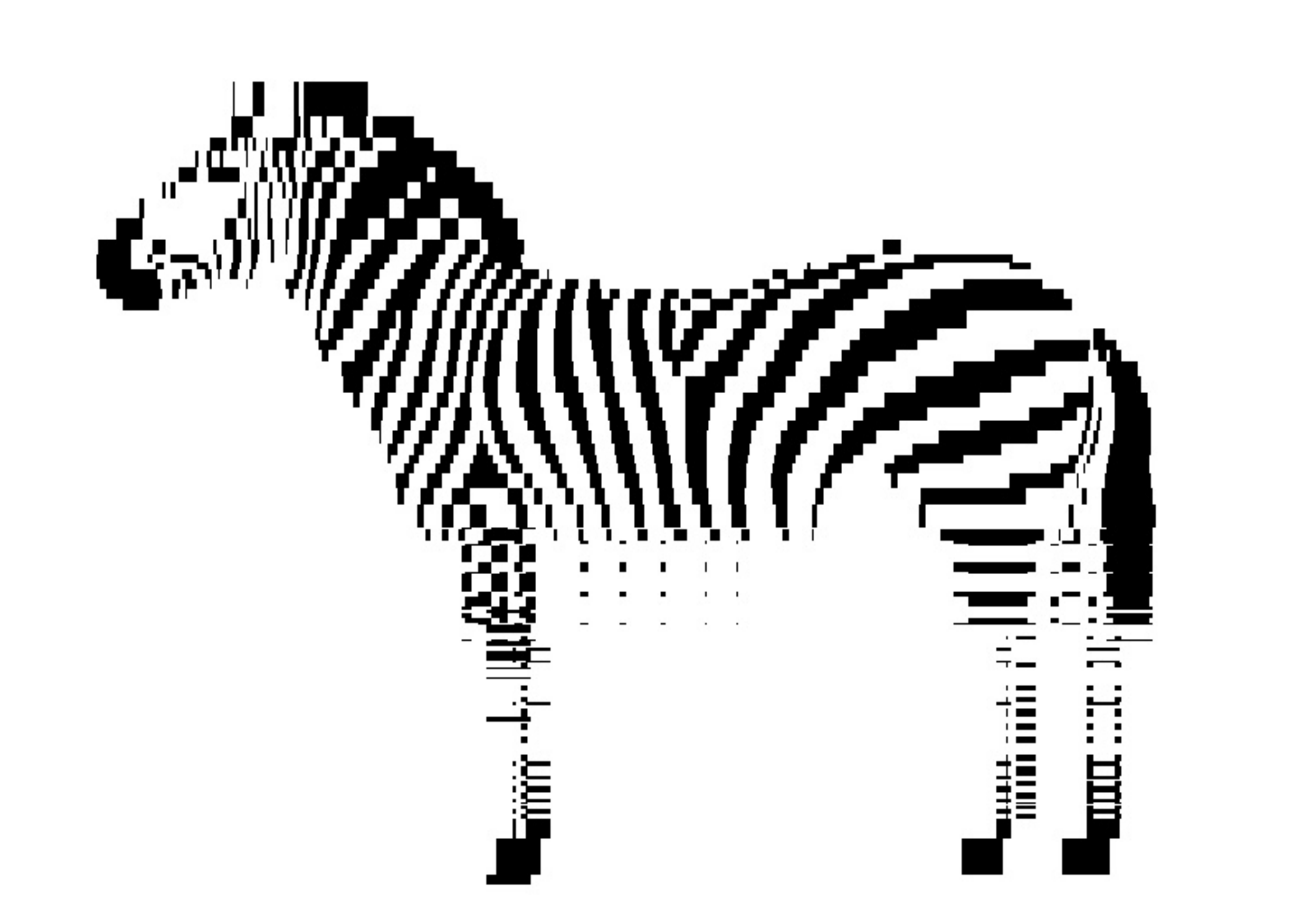} & \includegraphics[scale=0.08]{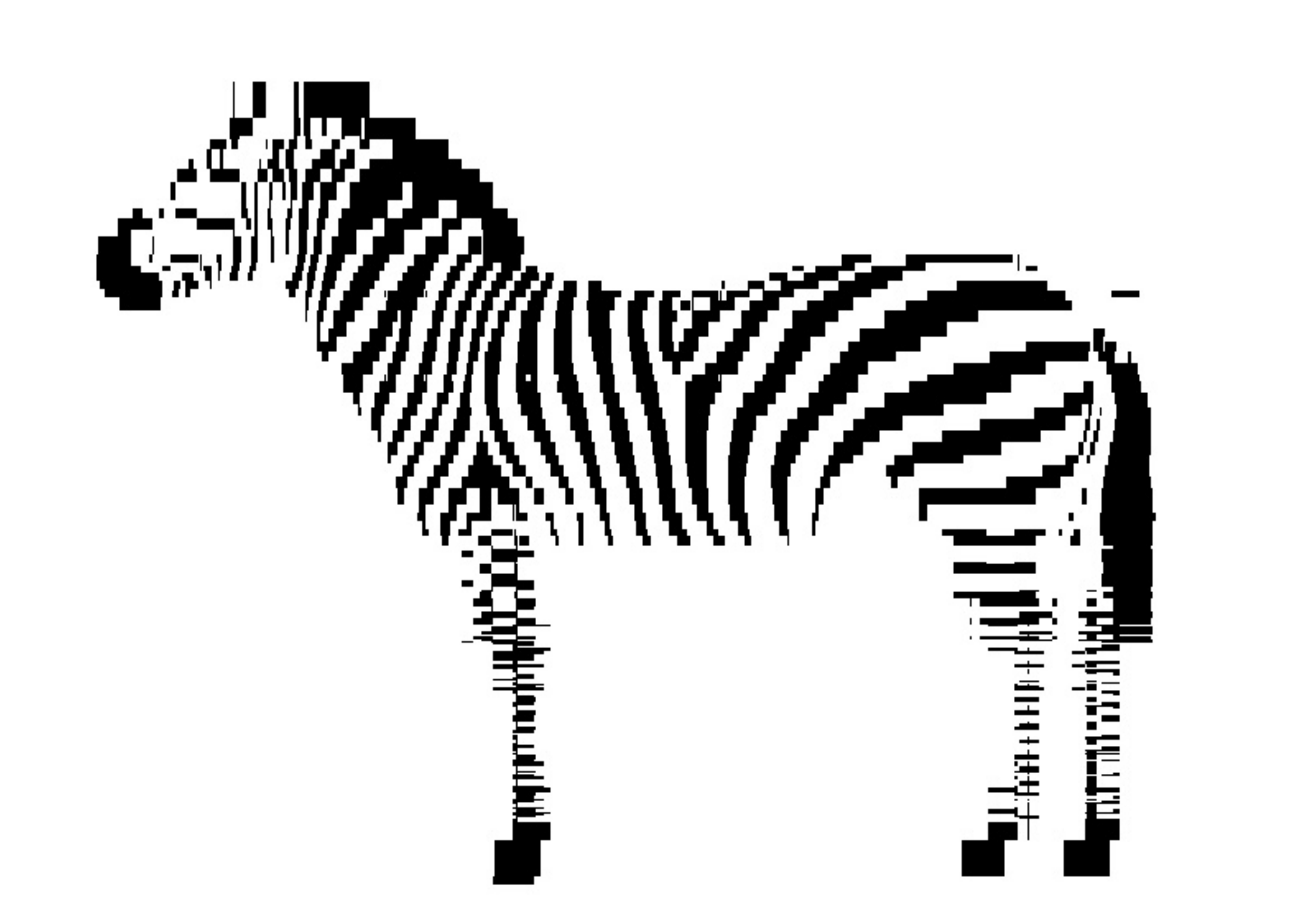} \\
&Error:  34414 & Error: 17502 & Error: 15723\\


          \midrule
\rotatebox{90}{Rank: 50}&           \includegraphics[scale=0.08]{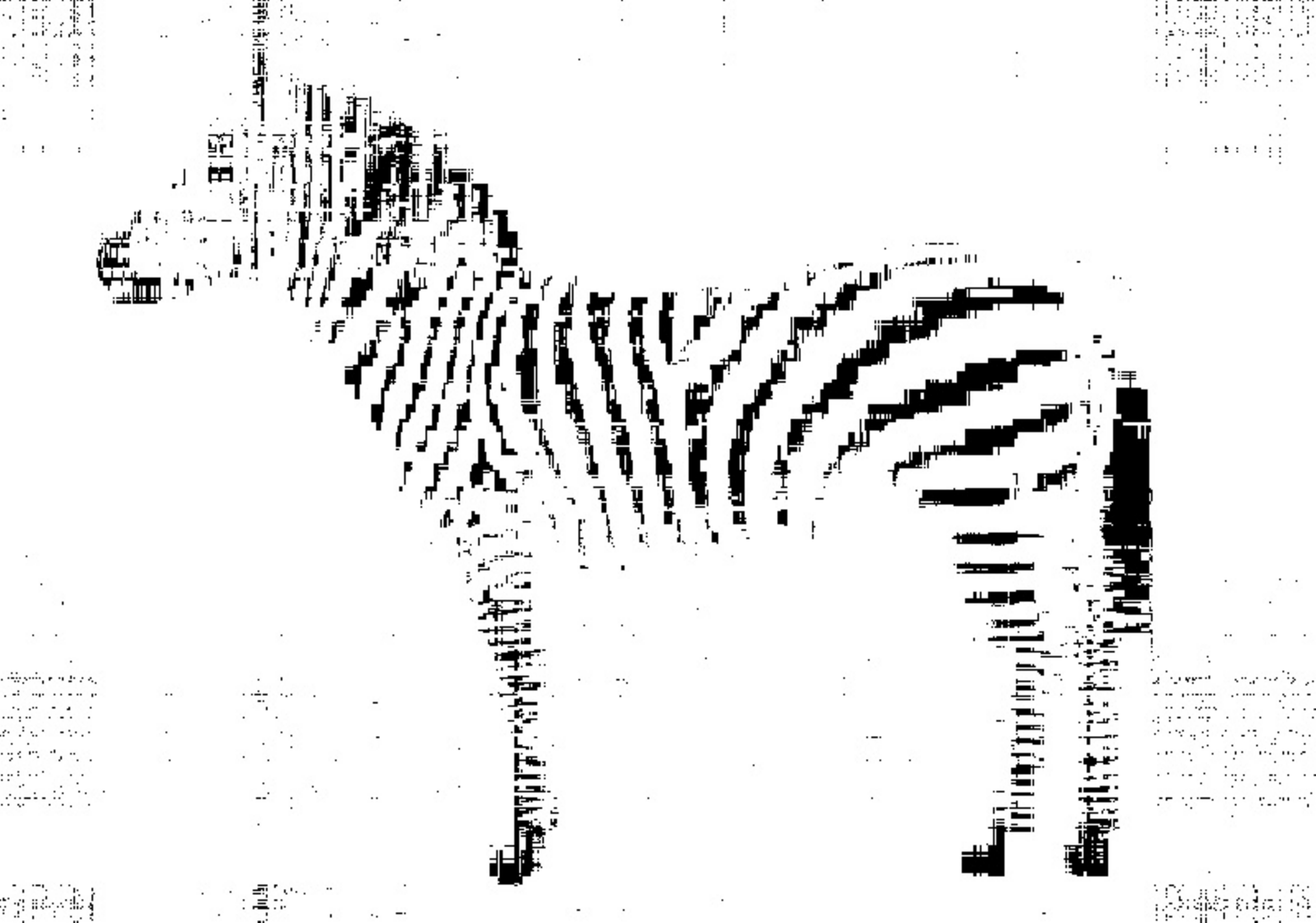} & \includegraphics[scale=0.08]{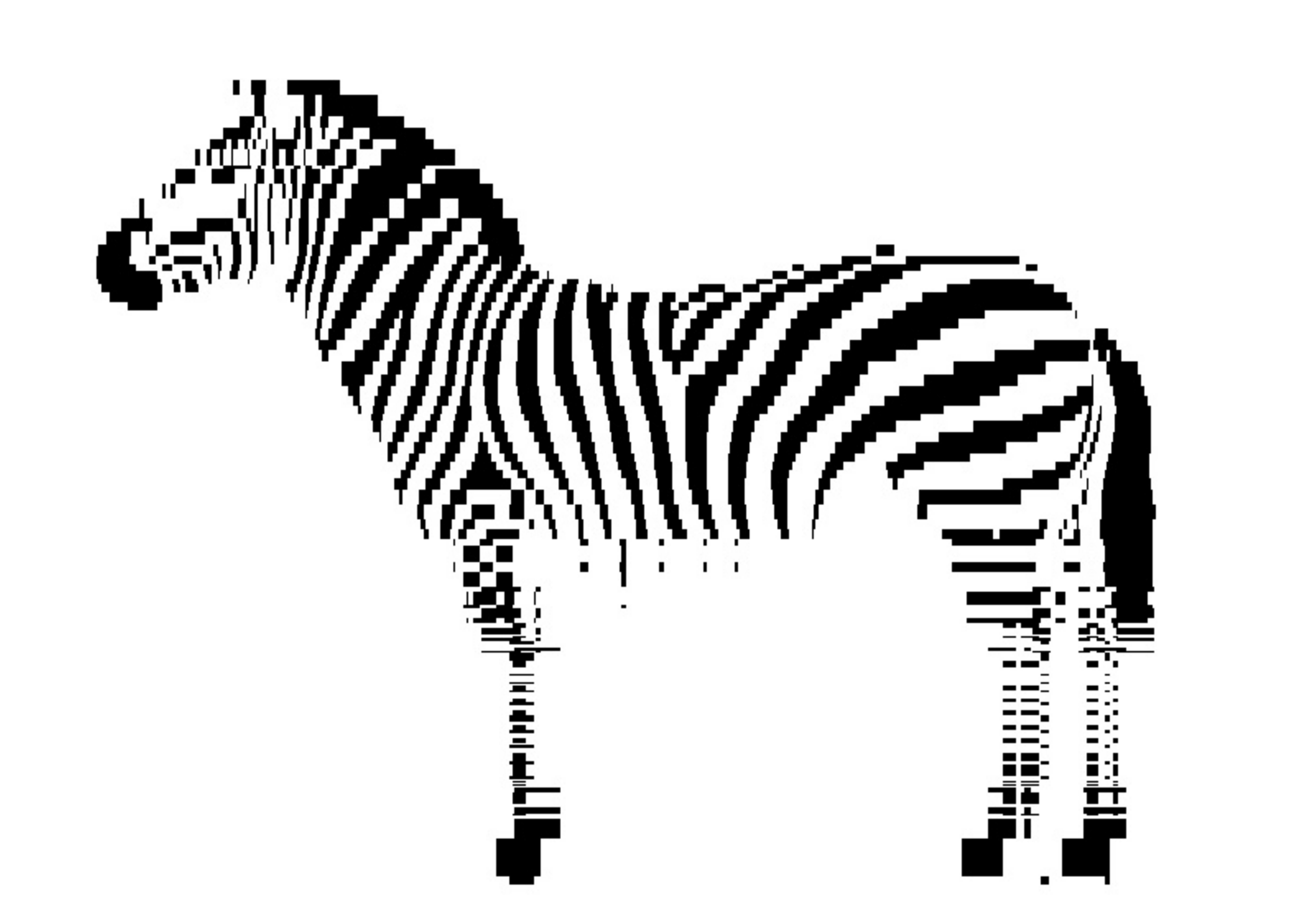} & \includegraphics[scale=0.08]{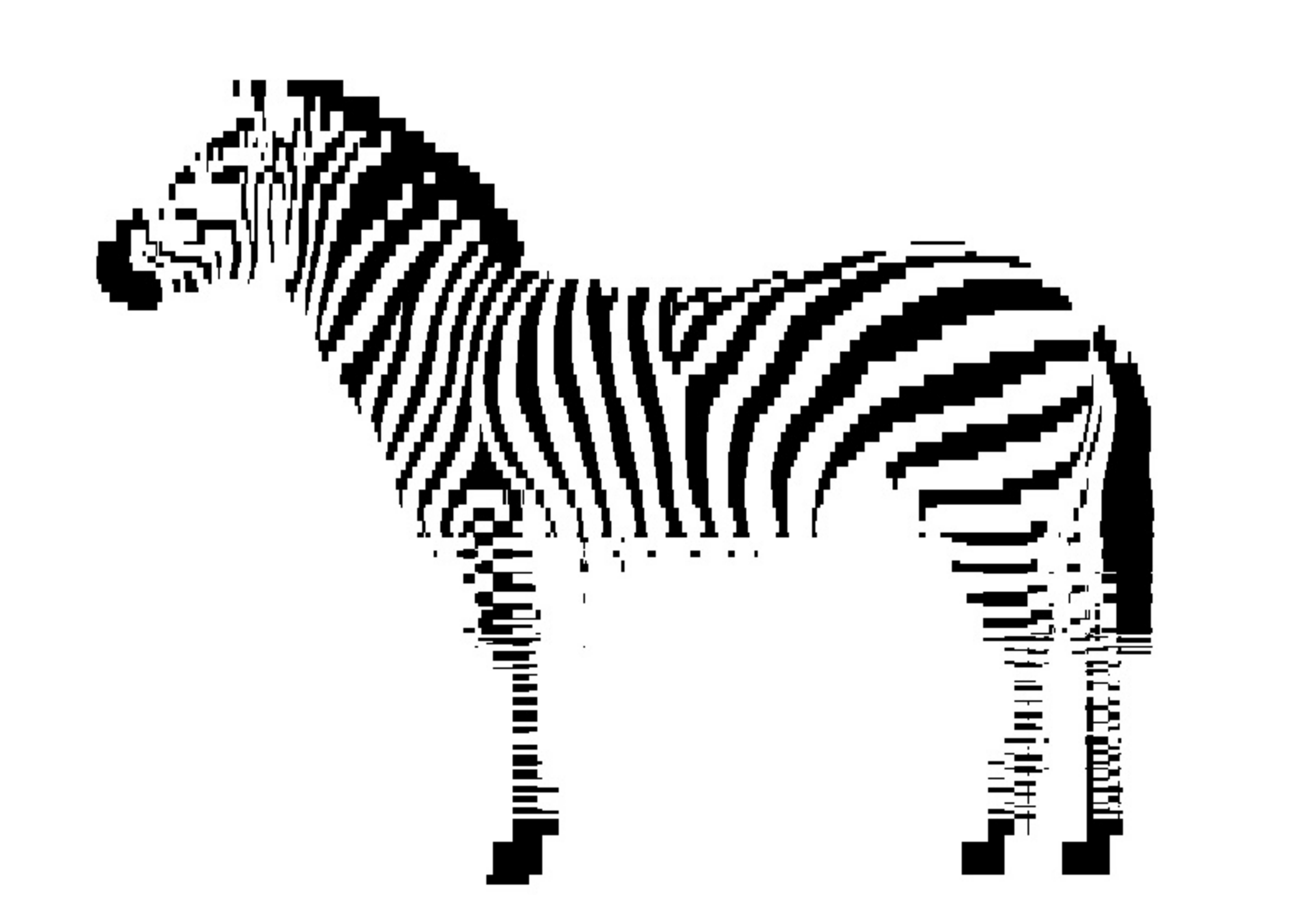} \\
& Error:  34445 & Error: 14974 & Error: 13684\\


          \midrule
   \rotatebox{90}{Rank: 100}&          \includegraphics[scale=0.08]{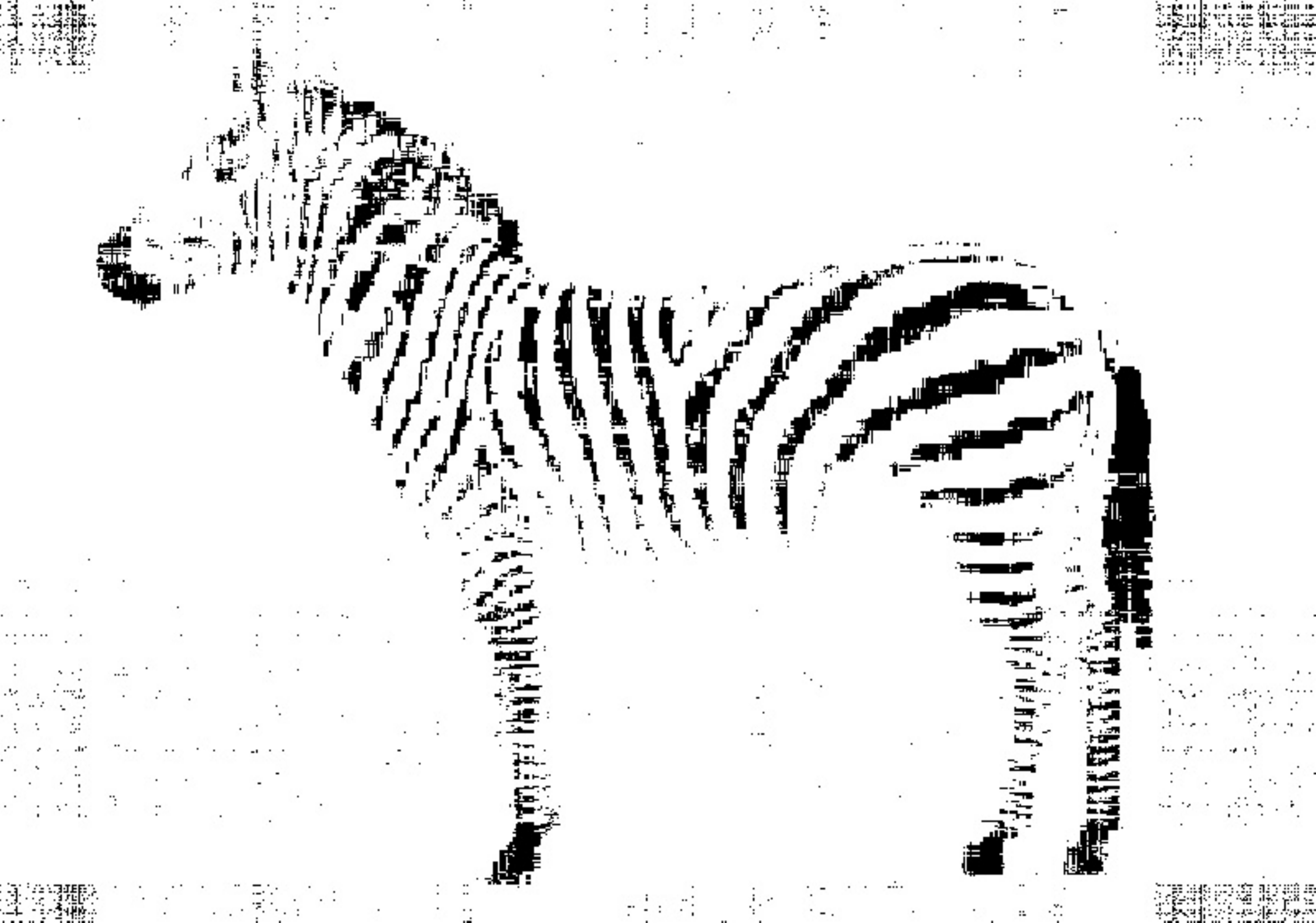} & \includegraphics[scale=0.08]{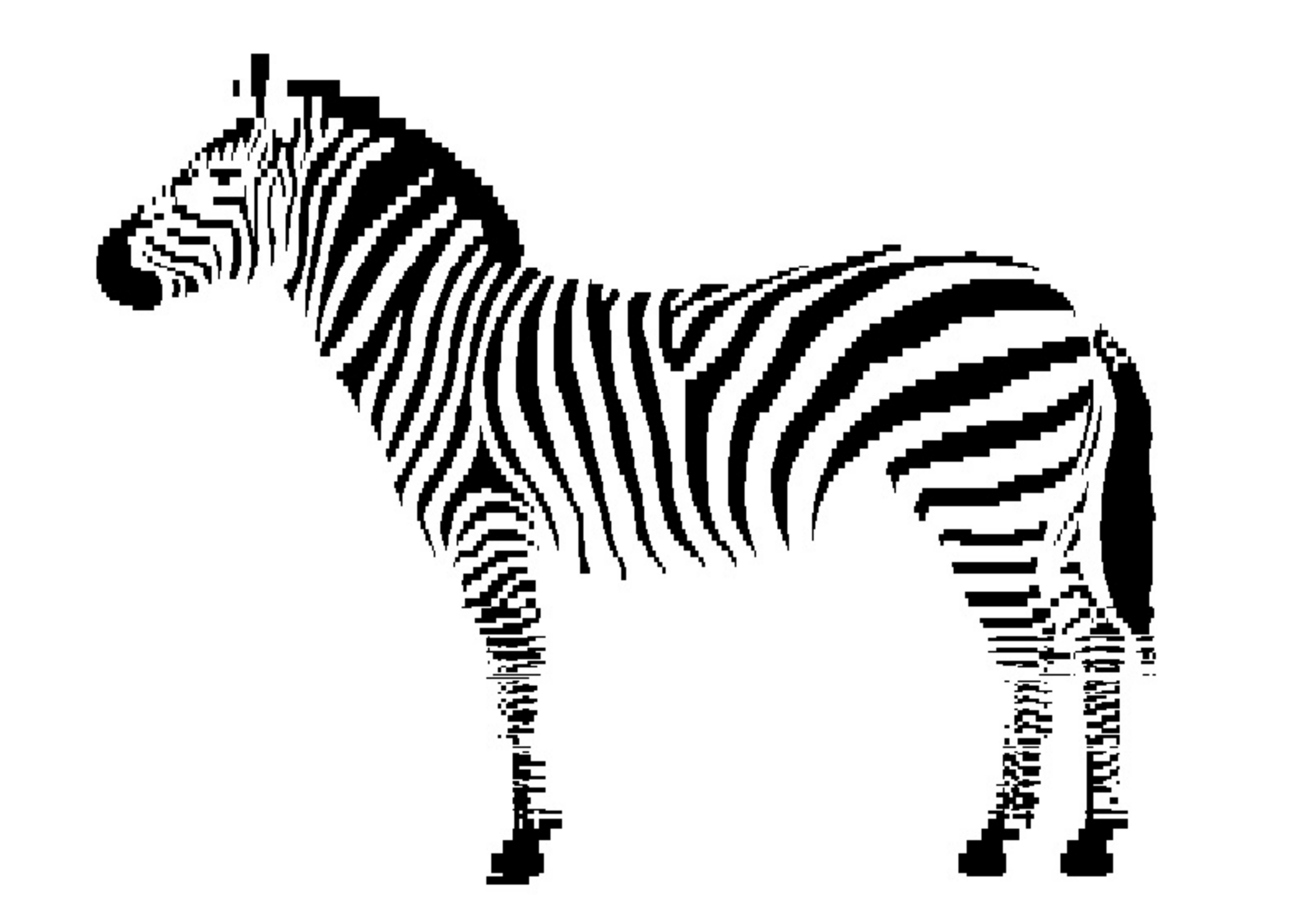} & \includegraphics[scale=0.08]{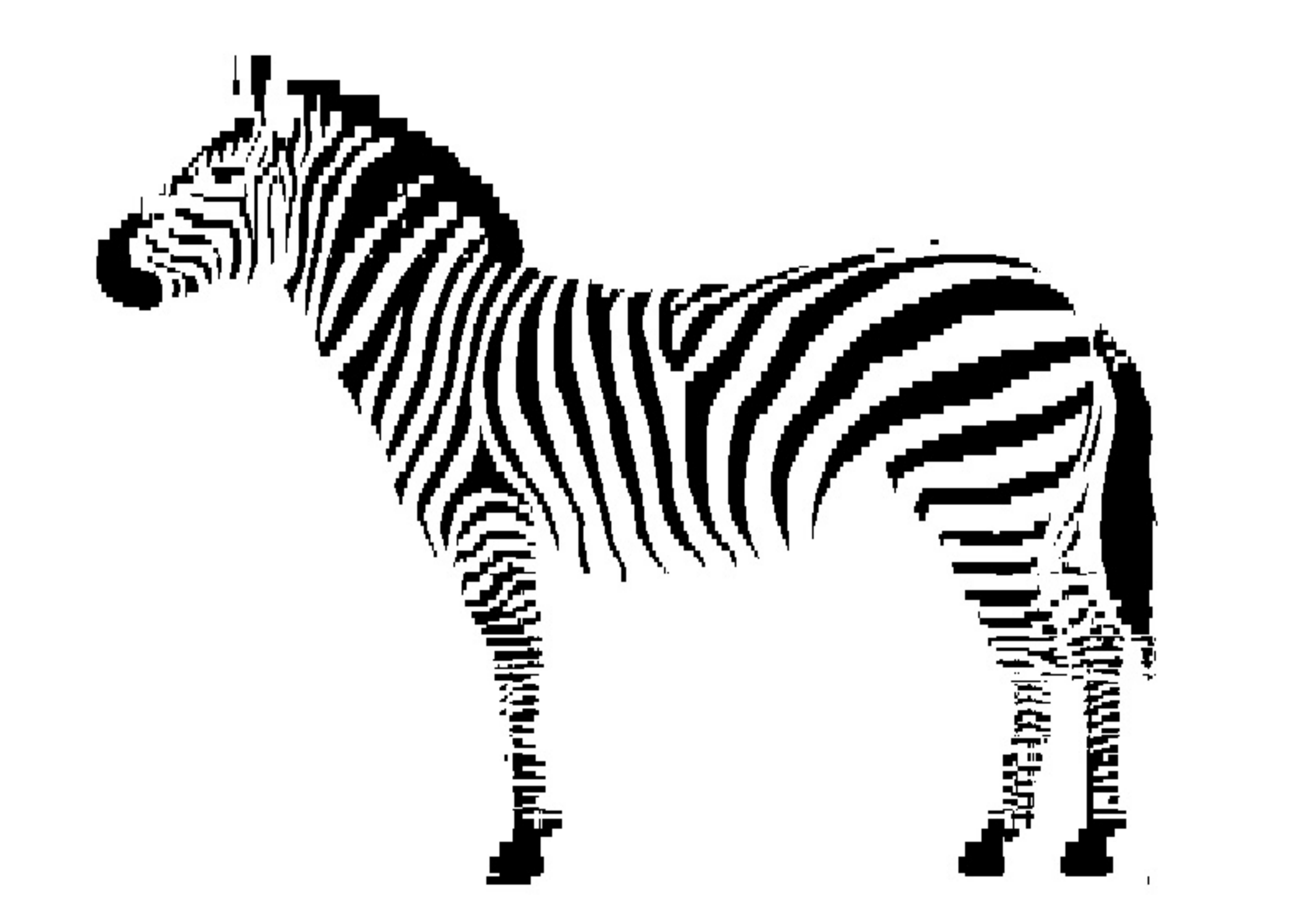} \\
& Error: 33445 & Error: 8709 & Error: 8529\\          
            \bottomrule
        \end{tabular}
        }
    \end{table}

\begin{figure}[t]
\centering
\includegraphics[width = 8cm]{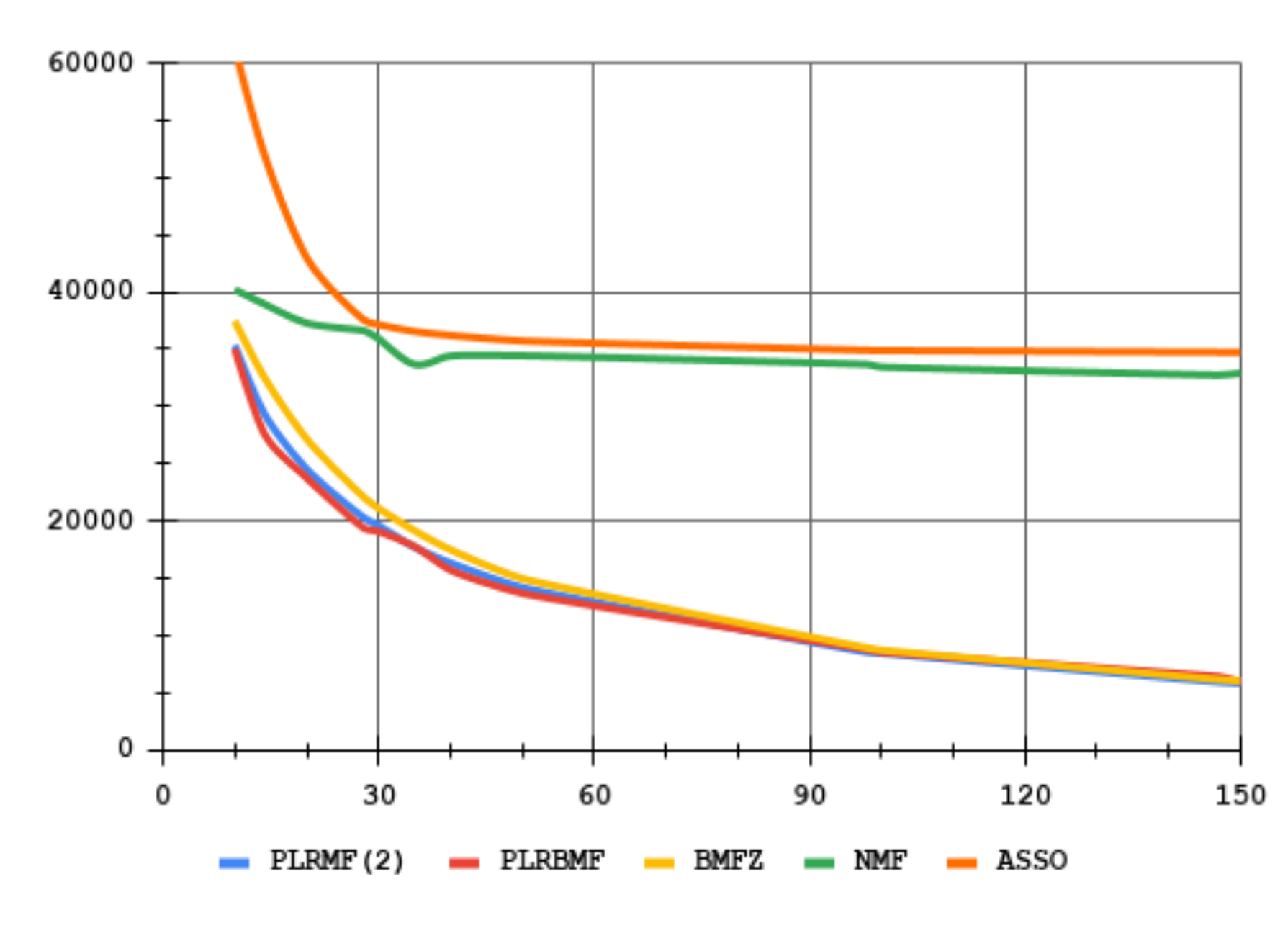}
\caption{Graph depicting performance of our algorithms compared with others on the image mentioned in Table~\ref{tbl:zebra}.}
\label{graph:zebra}
\end{figure}

We analyse performance of our algorithms on binary and gray scale images. 
Table~\ref{tbl:zebra} shows the performance of {\sf PLRBMF} compared with {\sf NMF} and {\sf BMFZ}. We would like to mention that both our algorithms {\sf PLRBMF} and {\sf PLRMF($2$)} work better than {\sf Asso}, {\sf NMF}, and {\sf BMFZ}. Here, we included results of {\sf PLRBMF}, {\sf NMF}, and {\sf BMFZ}. For the ranks mentioned in the table, {\sf PLRBMF} performs better than {\sf PLRMF($2$)} and both these algorithms perform better than the other algorithms mentioned here. For the inputs in Table~\ref{tbl:zebra}, {\sf NMF} and {\sf BMFZ} perform better than {\sf Asso}. So we compared {\sf PLRBMF} with {\sf NMF} and {\sf BMFZ}. 
The performance  of all the above algorithms are summarized in Figure~\ref{graph:zebra}. 
Notice that {\sf NMF} gives two no-negative real matrices $\bfU$ and $\bfV$. We round the values in these matrices to $0$ and $1$ by choosing a best possible threshold that minimizes error in terms of $\ell_1$-norm.  After rounding the values in the matrices $\bfU$ and $\bfV$ we get two  binary matrices $\bfU'$ and $\bfV'$.  Then we  multiply $\bfU'$ and $\bfV'$ in GF(2) to get the output matrix.

In  Table \ref{tbl:mri}, we have taken an MRI Grayscale image, which consists of only 7 shades, as an input matrix whose dimensions are 266 $\times$ 247. This image is obtained by changing values between $0-255$ to $7$ distinct values of an image from \cite{brainimage}. 
For running {\sf PLRMF(7)} on this image, we mapped those 7 shades to the numbers 0, 1, …, 6 so that we can get a matrix in {\sf GF(7)}, then we run {\sf PLRMF(7)} on the modified matrix and again remapped the entries of output matrix to their respective shades, thereby, getting an output image and then we calculated error (sum of absolute errors) between input matrix and output matrix. 
For {\sf NMF} we took the original matrix as the input (i.e., the matrix with values from $0-255$). Also, since {\sf NMF} gives us a matrix with real values, we have rounded the matrix values to the nearest integer and called it the output matrix and then calculated error. 
We can see clearly from Table \ref{tbl:mri} and Figure~\ref{graph:mri} that {\sf PLRMF(7)}  which works only on finite fields is performing far better than {\sf NMF}.
It is important to note here that the input rank used in both the algorithms {\sf PLRMF(7)} and {\sf NMF} in Table \ref{tbl:mri} varies from 20 to 100. However, since the output of {\sf NMF} algorithm is two real-valued factor matrices $\bfU$ and $\bfV$, and to get the output image the real values in $\bfU \bfV$ were rounded up to the nearest integer. Because of this, the rank of the output matrix is altered which is mentioned below the images in {\sf NMF} column as Real Rank. Due to the increase in rank and the values in the output matrix of {\sf NMF} can have much more than 7 distinct values, the images under the {\sf NMF} column look better despite having higher error. 
To get the bounded rank output by the method of {\sf NMF},  when we  round the elements of the factor matrices to the nearest integer,  the output matrix has all values zeros, resulting in a image with all pixels black.  


\begin{figure}[t]
\centering
\includegraphics[width = 8cm]{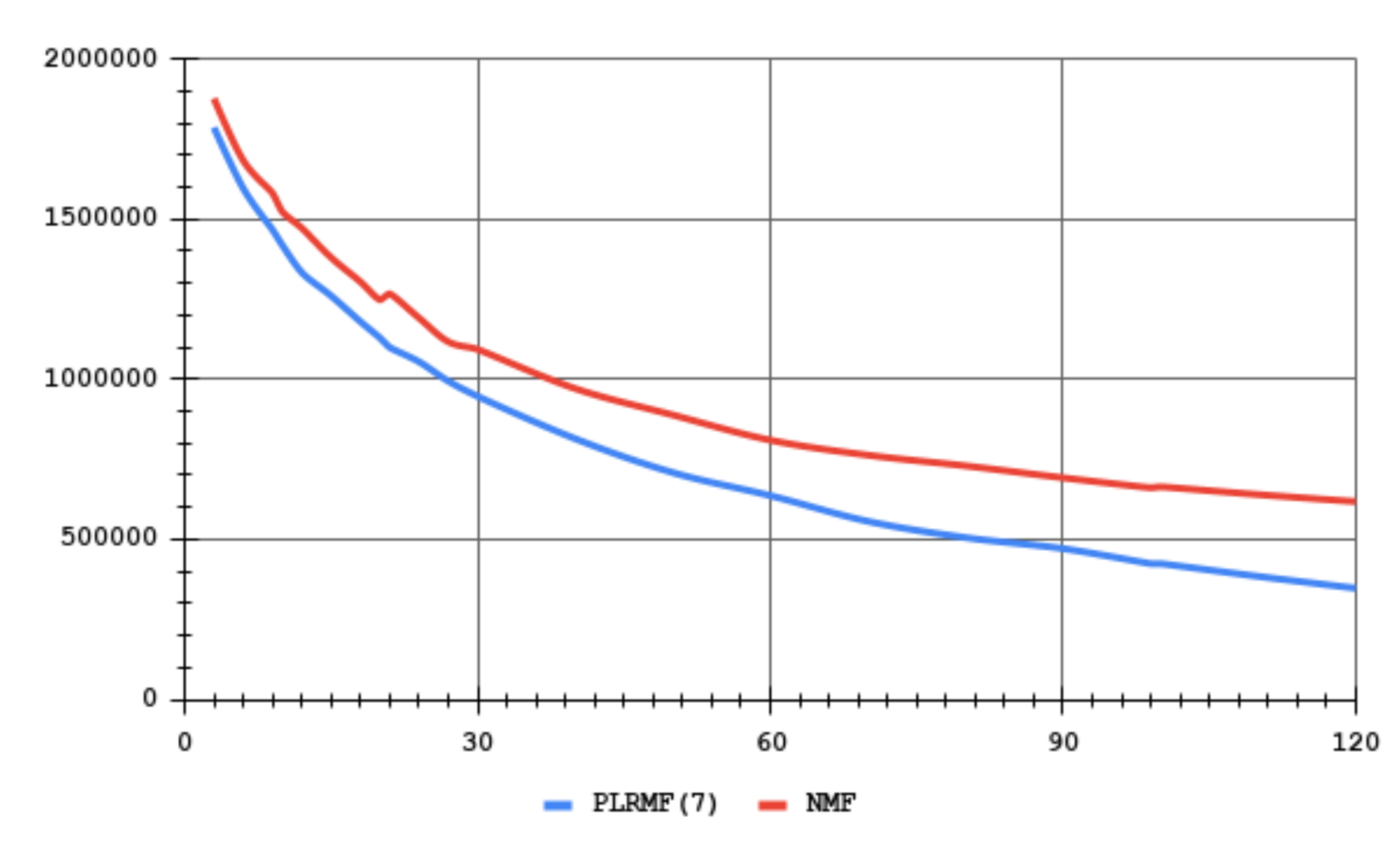}
\caption{Graph depicting performance of our algorithm compared with {\sf NMF} on the image mentioned in Table~\ref{tbl:mri}.}
\label{graph:mri}
\end{figure}

	\begin{table}[H]

	        \caption{Performance of our algorithm {\sf PLRMF(7)} compared to {\sf NMF}. The dimension of the image is $266\times 247$. Here Alg. Rank is the input parameter $r$ to the algorithms. 
        }
        \label{tbl:mri}

		\centering
		\begin{tabular}{|c|c|c|}
			\toprule
			&{\sf PLRMF(7)} & {\sf NMF}  \\
			\midrule
    \rotatebox{90}{Original image}&          \includegraphics[scale=0.275]{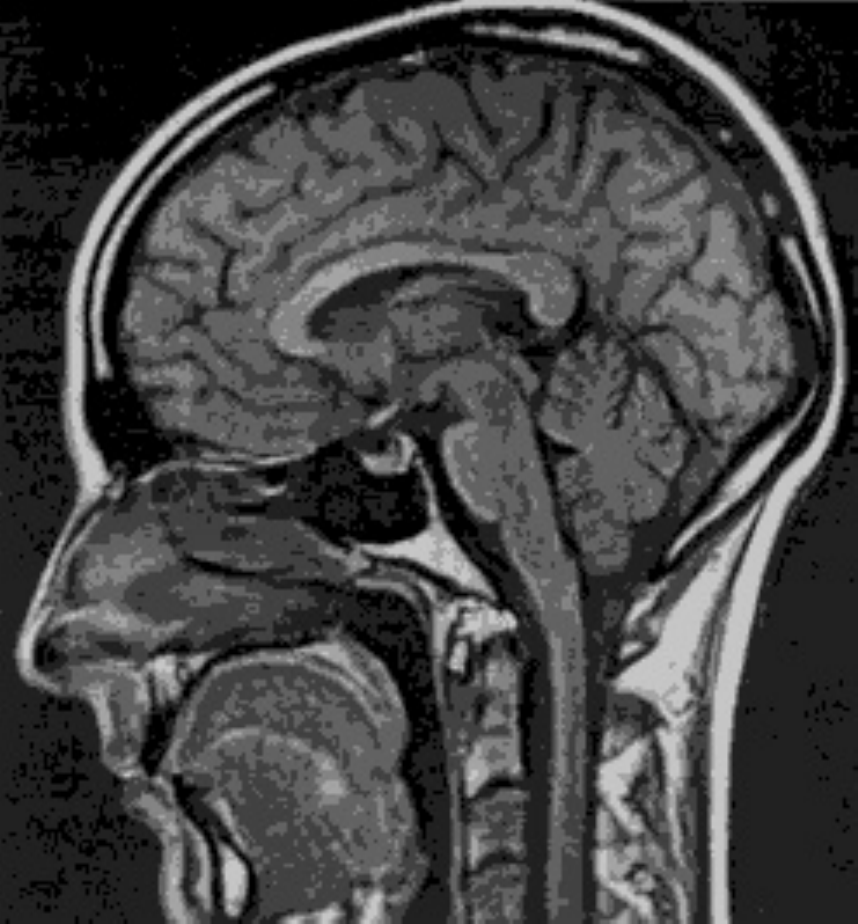} & \includegraphics[scale=0.275]{original_7_color_image} \\
          \midrule

   \rotatebox{90}{Alg. Rank: 10}&         \includegraphics[scale=0.275]{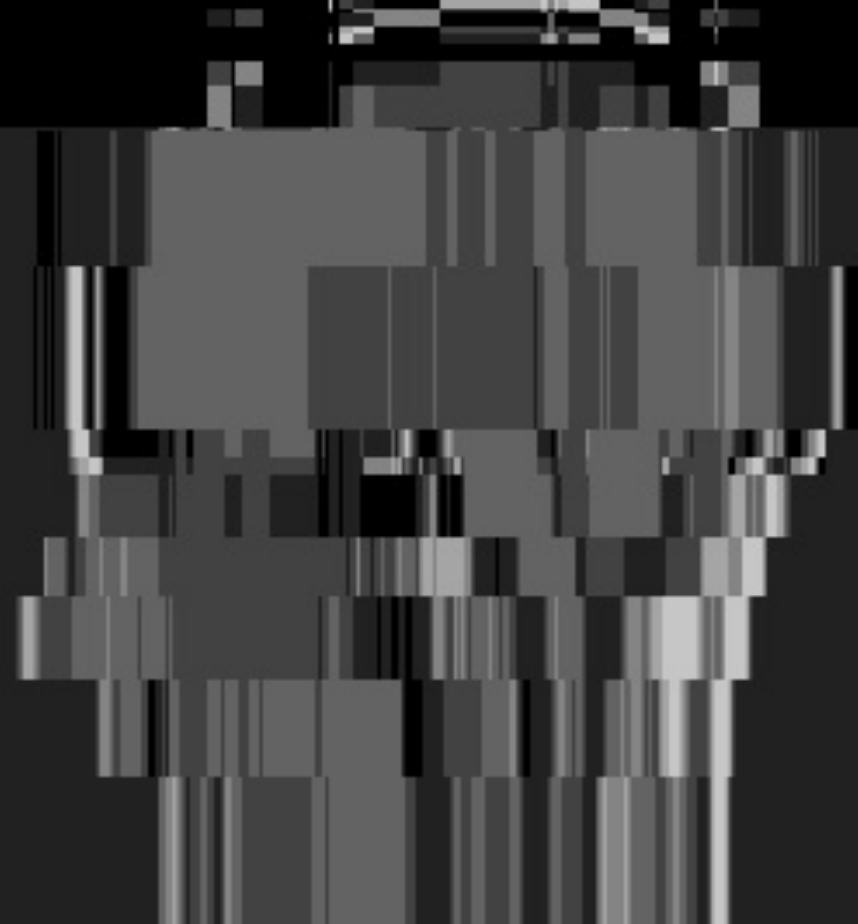}  & \includegraphics[scale=0.275]{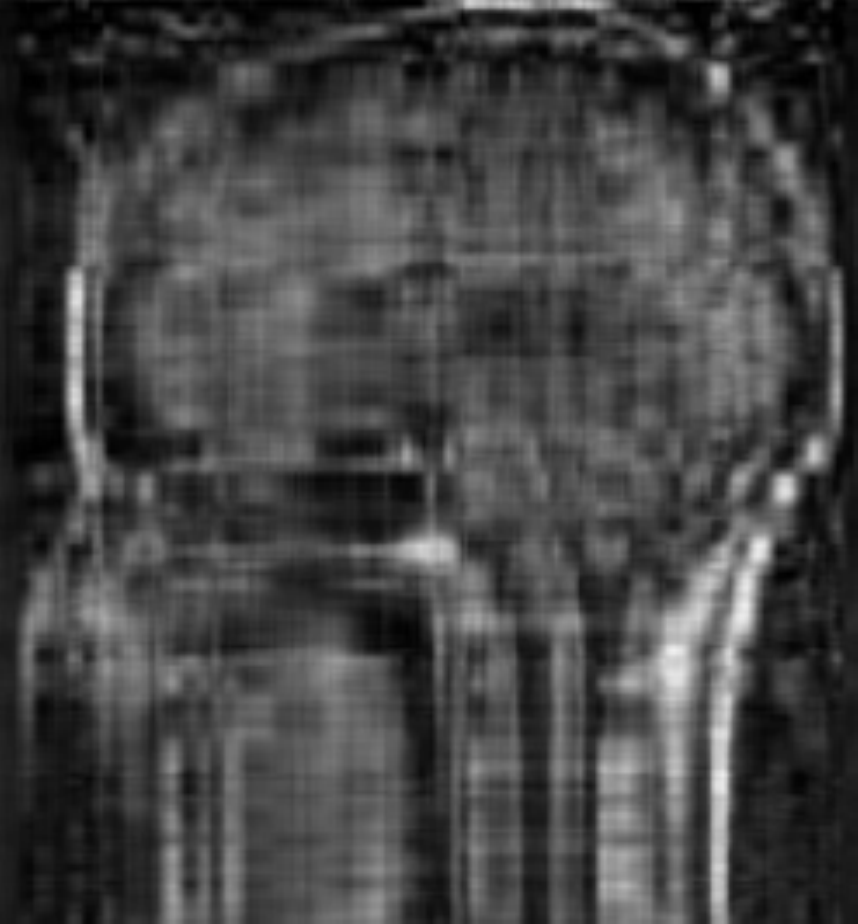} \\
& Error: 1419231 & Error: 1523339 \\
& GF(7) Rank : 10 & Real Rank : 247 \\

       \midrule
\rotatebox{90}{Alg. Rank: 20 }&           \includegraphics[scale=0.275]{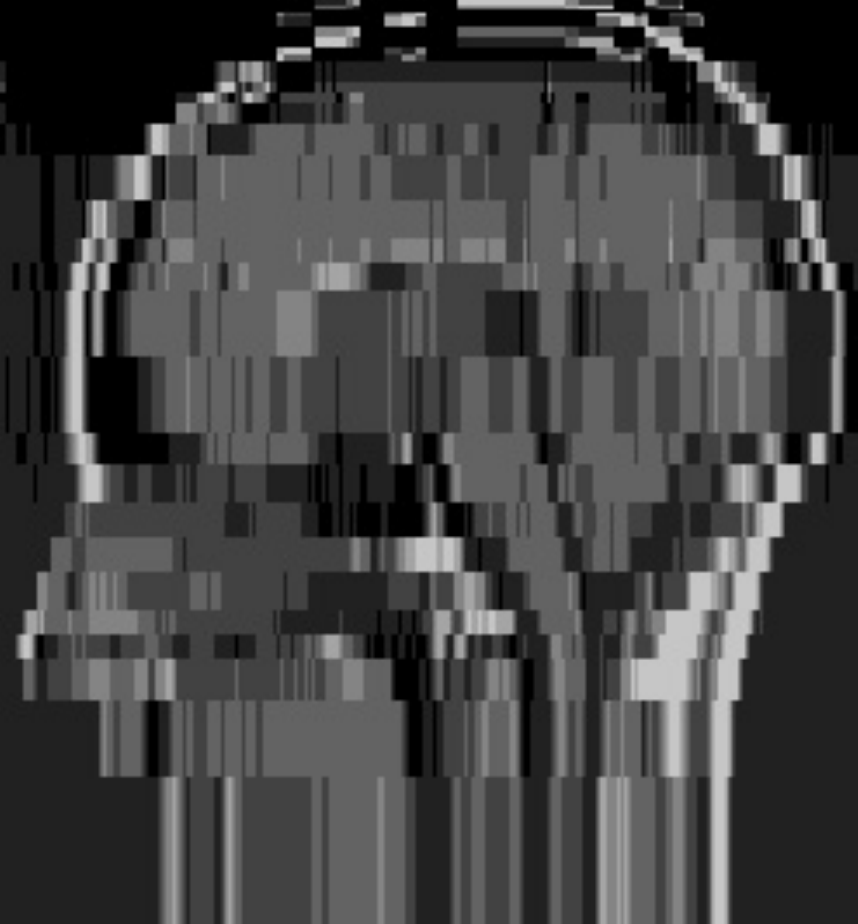} & \includegraphics[scale=0.275]{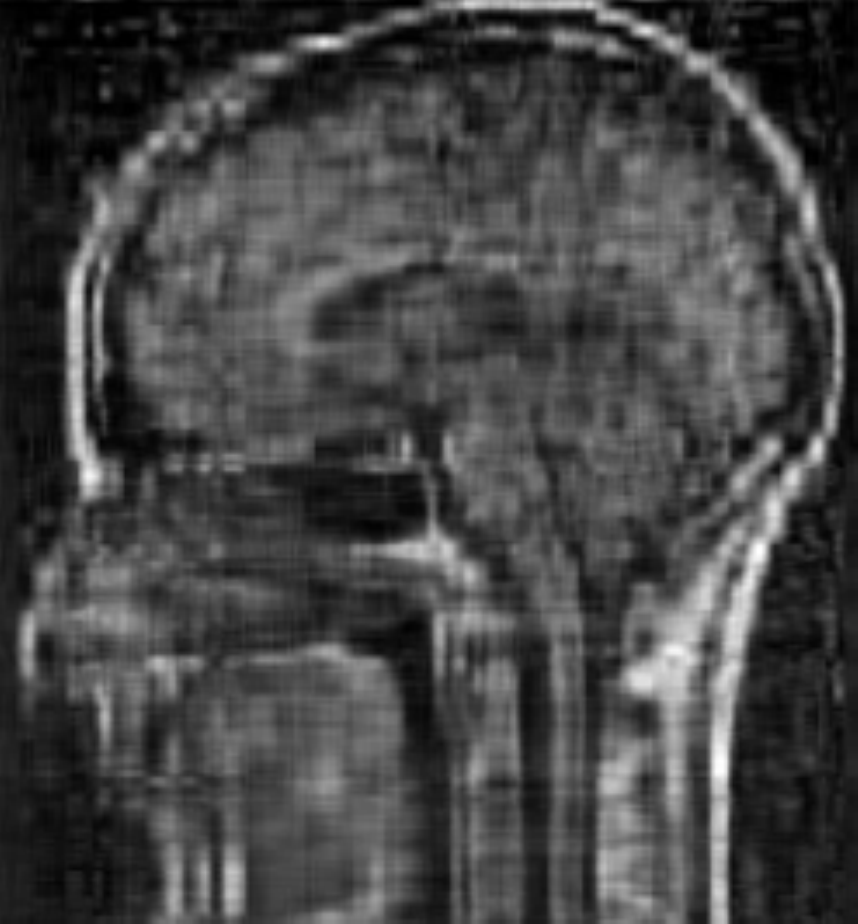} \\
& Error: 1129755 & Error: 1248532 \\
& GF(7) Rank : 20 & Real Rank : 247 \\
          \midrule
\rotatebox{90}{Alg. Rank: 30 }&        \includegraphics[scale=0.275]{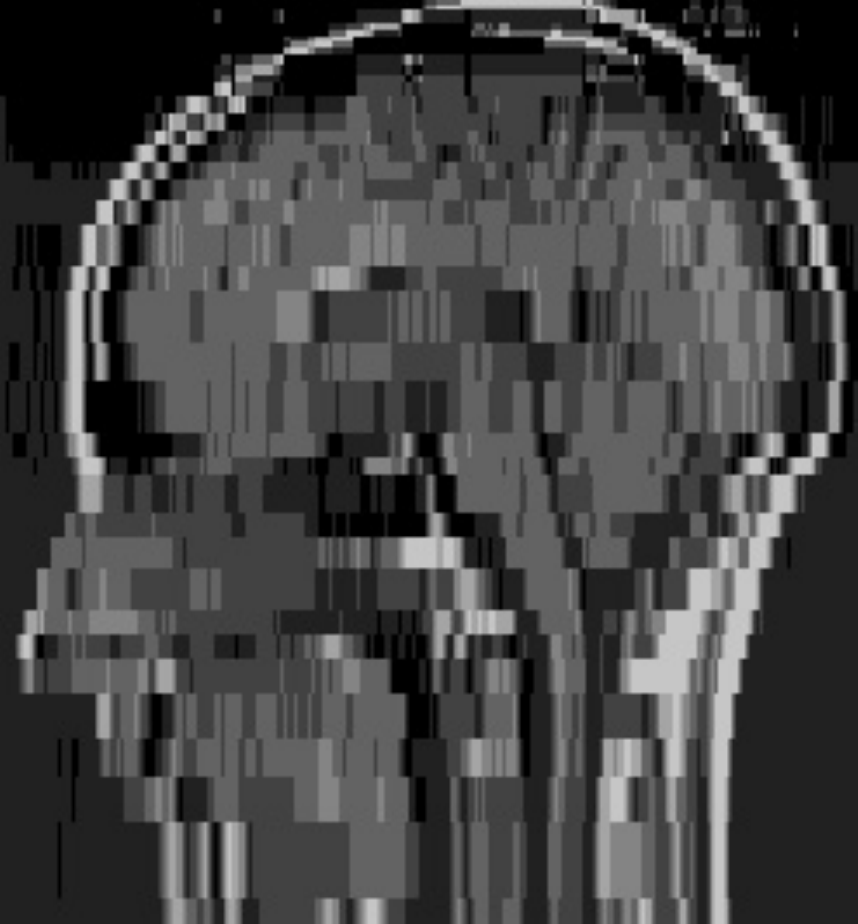} & \includegraphics[scale=0.275]{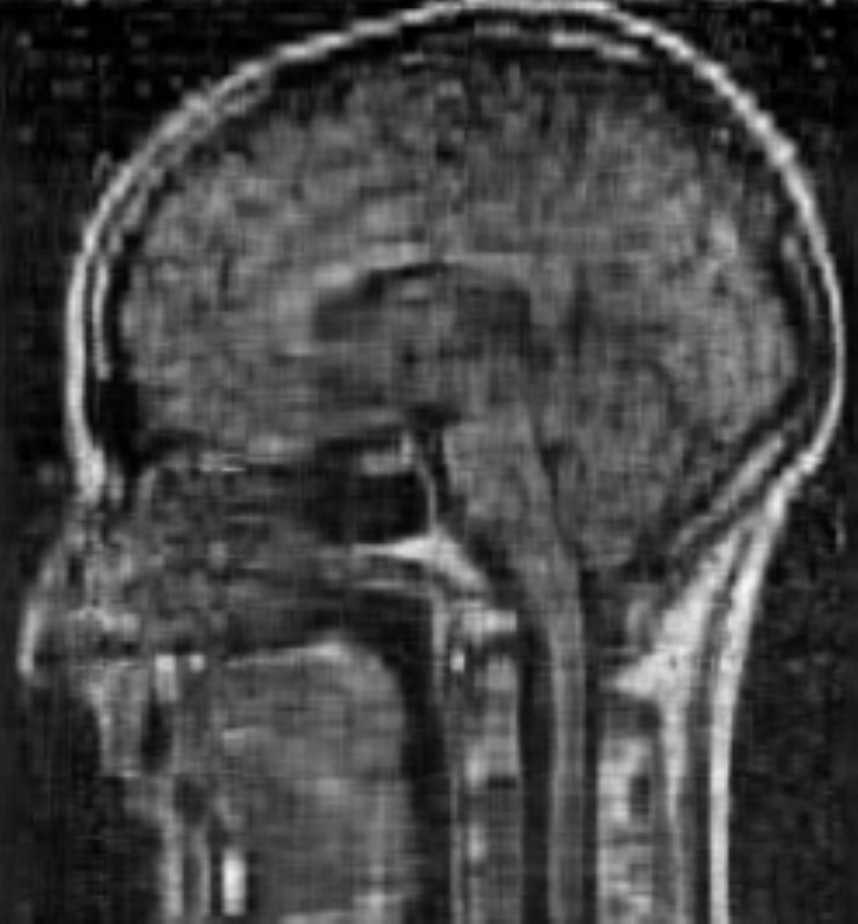} \\
& Error: 947298 & Error: 1093683 \\
& GF(7) Rank : 30 & Real Rank : 247 \\


%

          \midrule
\rotatebox{90}{Alg. Rank: 50}&           \includegraphics[scale=0.275]{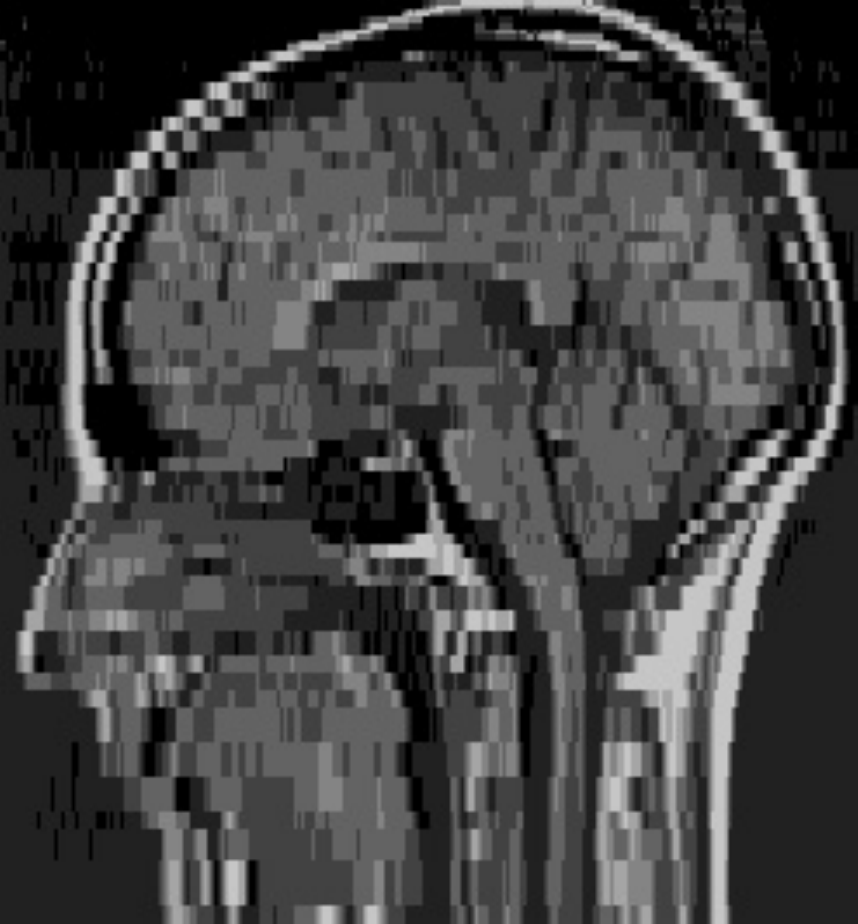} & \includegraphics[scale=0.275]{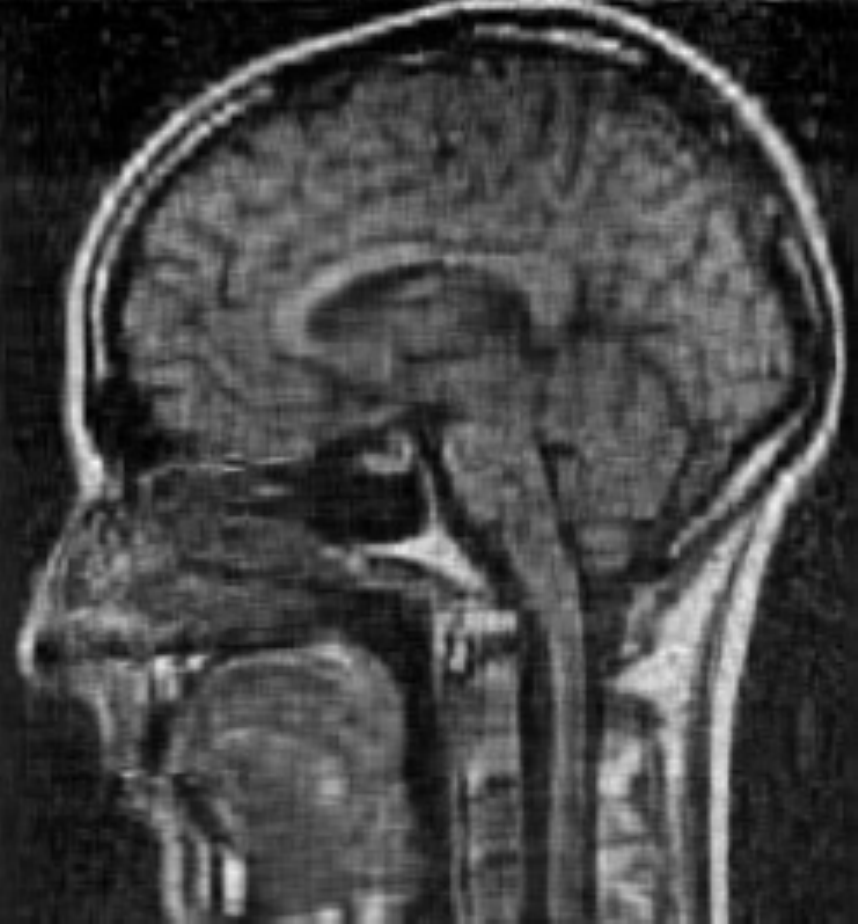} \\
& Error: 709203 & Error: 888976  \\
& GF(7) Rank : 50 & Real Rank : 247 \\

          \midrule
   \rotatebox{90}{Alg. Rank: 100}&          \includegraphics[scale=0.275]{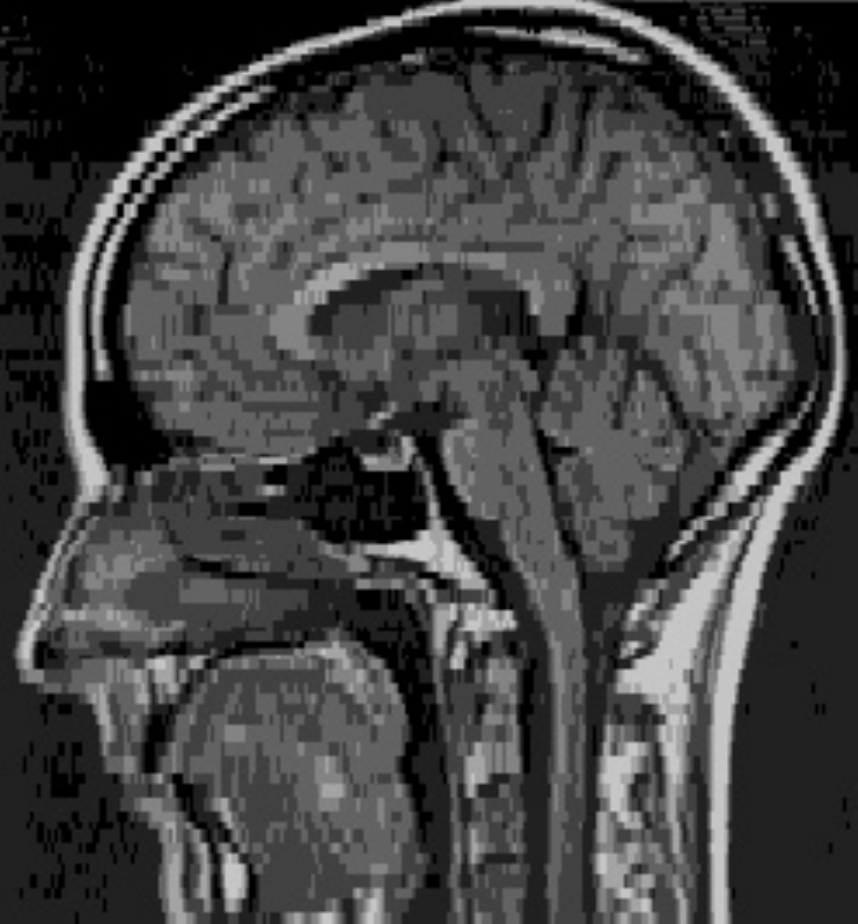} & \includegraphics[scale=0.275]{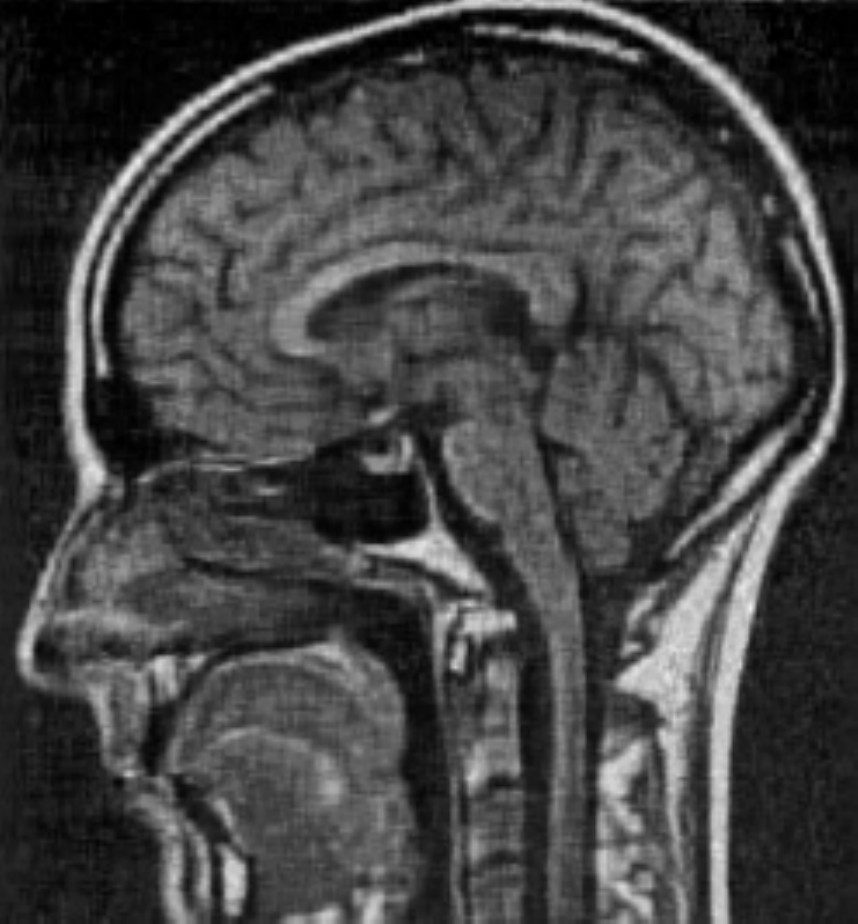} \\
& Error: 424776 & Error: 664537 \\
& GF(7) Rank : 100 & Real Rank : 247 \\
            \bottomrule
        \end{tabular}
	\end{table}


\begin{table}[t]
    \caption{Comparison on movie-lens data where the ranks considered are between $1$ and $30$, on matrix of dimension $43\times 134$.}
    \label{tablemovielens}
    \centering
      \begin{tabular}{|c|c|c|c|c|c|c|}
    \specialrule{.2em}{.1em}{.1em}
    Rank     & 1     & 2     & 3     & 6     \\
    \specialrule{.2em}{.1em}{.1em}
    {\sf PLRMF($11$)} &
    4981 &	4527 &	4273 &	3791  \\
    \hline
    {\sf NMF} & 5257.8 &	5201 &	5015 &	4652.7 \\
    \specialrule{.2em}{.1em}{.1em}
\end{tabular}

      \begin{tabular}{|c|c|c|c|c|}
    \specialrule{.2em}{.1em}{.1em}
    Rank     & 9  & 12    & 15 & 18         \\
    \specialrule{.2em}{.1em}{.1em}
    {\sf PLRMF($11$)} &	3422 &	3070 &
    2935 & 2477 \\
    \hline
    {\sf NMF} & 3924 &	3628	& 4305.6 &	4066.4\\
    \specialrule{.2em}{.1em}{.1em}
\end{tabular}

     \begin{tabular}{|c|c|c|c|c|}
    \specialrule{.2em}{.1em}{.1em}
    Rank      &	 21 & 24 & 27 &30           \\
    \specialrule{.2em}{.1em}{.1em}
    {\sf PLRMF($11$)}  &	2145 &	1935 &	1569 &	1129 \\
    \hline
    {\sf NMF}  &	3556 &	3336 &	3090 &	2982 \\
    \specialrule{.2em}{.1em}{.1em}
\end{tabular}

\end{table}


\begin{figure}
\centering
\includegraphics[width = 8cm]{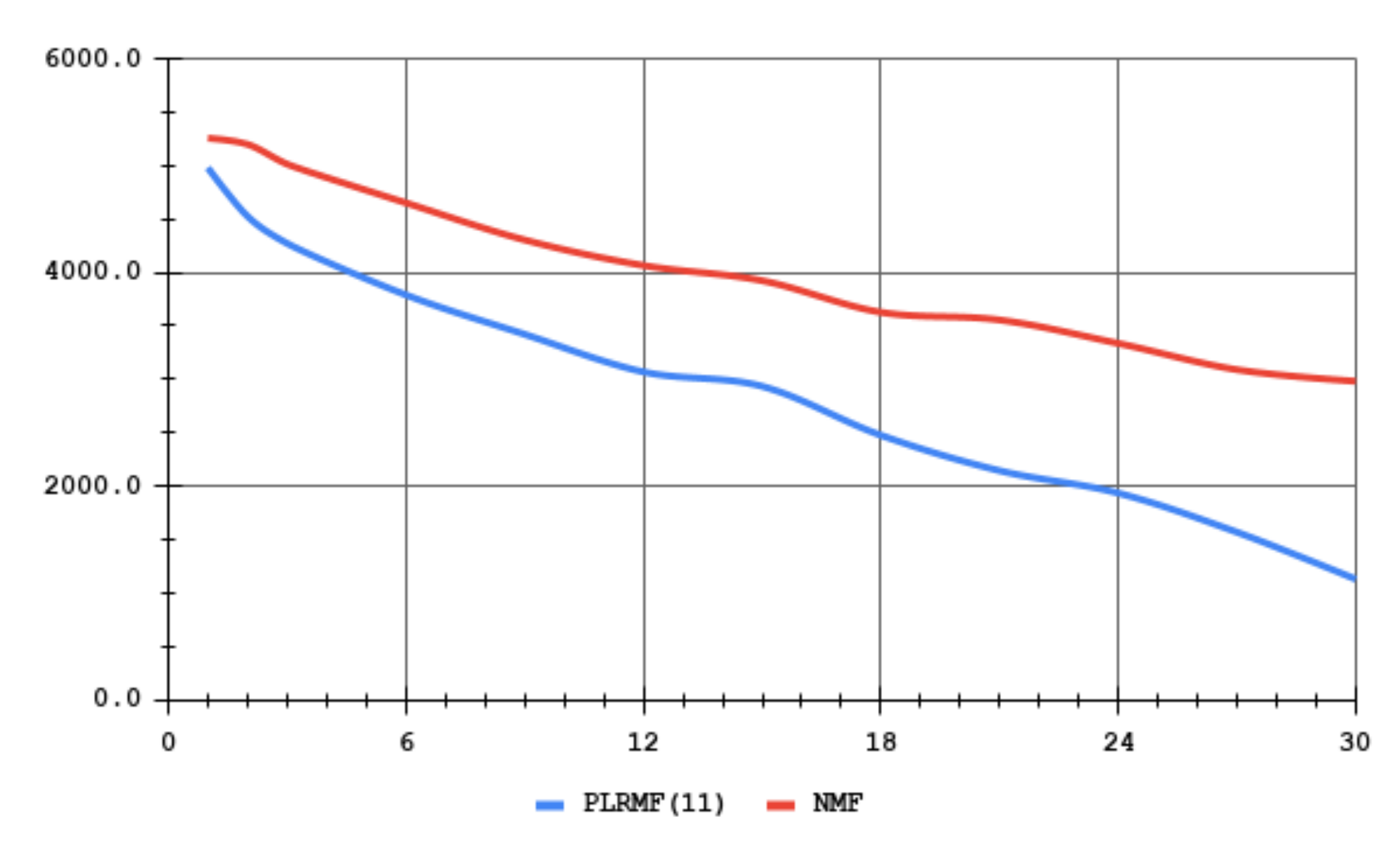}
\caption{Graph on real data where the entries in the table are entries in user vs movie matrix of movie lens data set where each user has rated at least 200 movies and each movie is rated by at least 150 users. The dimension of the matrix is $43\times 134$.}
\label{fig:movielens}
\end{figure}

We analyze {\sf PLRMF($11$)} on movielens data set~\cite{10.1145/2827872} and compare it with {\sf NMF} 
and 
the performance can be found in Figure~\ref{fig:movielens} and Table~\ref{tablemovielens}. 
The performance of {\sf PLRMF($11$)} against {\sf NMF}, monotonically increasing with respect to rank. 
We obtain more than $14\%$ improvement on rank $3$,  more than $24\%$ improvement on rank $12$, and 
more than $39\%$ improvement on rank $21$ against {\sf NMF}.